%% file: MRUC.tex
\documentclass{article}

\makeatletter
\def\input@path{{source/arxiv_source/}}
\makeatother

\usepackage{arxiv}
\input{command.tex}

\title{LOW-RANK MATRIX RECOVERY  WITH UNKNOWN CORRESPONDENCE}

\author{ \href{https://orcid.org/0000-0000-0000-0000}{\includegraphics[scale=0.06]{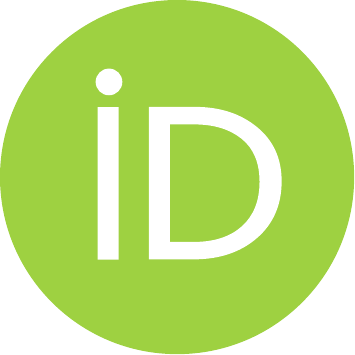}\hspace{1mm}Zhiwei Tang}\thanks{Zhiwei Tang, Tsung-Hui Chang and Hongyuan Zha are also affiliated with  Shenzhen Research Institute of Big Data} \\
The Chinese University of Hong Kong, Shenzhen\\
	\texttt{zhiweitang1@link.cuhk.edu.cn} \\
	\And
	\href{https://orcid.org/0000-0000-0000-0000}{\includegraphics[scale=0.06]{orcid.pdf}\hspace{1mm}Tsung-Hui Chang}\footnotemark[1] \\
	The Chinese University of Hong Kong, Shenzhen\\
	\texttt{changtsunghui@cuhk.edu.cn} \\
	\And
	\href{https://orcid.org/0000-0000-0000-0000}{\includegraphics[scale=0.06]{orcid.pdf}\hspace{1mm}Xiaojing Ye} \\
	Georgia State University\\
\texttt{xye.gsu.edu} \\
	\And
	\href{https://orcid.org/0000-0000-0000-0000}{\includegraphics[scale=0.06]{orcid.pdf}\hspace{1mm}Hongyuan Zha}\footnotemark[1] \\
	The Chinese University of Hong Kong, Shenzhen\\
\texttt{zhahy@cuhk.edu.cn} \\
}

\renewcommand{\shorttitle}{Low-rank Matrix Recovery  With Unknown Correspondence}

\hypersetup{
pdftitle={A template for the arxiv style},
pdfsubject={q-bio.NC, q-bio.QM},
pdfauthor={David S.~Hippocampus, Elias D.~Striatum},
pdfkeywords={First keyword, Second keyword, More},
}

\begin{document}
\maketitle

\begin{abstract}
	We study a matrix recovery problem with unknown correspondence: given the observation matrix $M_o=[A,\tilde P B]$, where $\tilde P$ is an unknown permutation matrix, we aim to recover the underlying matrix $M=[A,B]$. Such problem commonly arises in many applications where heterogeneous data are utilized and the correspondence among them are unknown, e.g., due to privacy concerns. We show that it is possible to recover $M$ via solving a nuclear norm minimization problem under a proper low-rank condition on $M$, with provable non-asymptotic error bound for the recovery of $M$. We propose an algorithm, $\text{M}^3\text{O}$ (Matrix recovery via Min-Max Optimization) which recasts this combinatorial problem as a continuous minimax optimization problem and solves it by proximal gradient with a Max-Oracle. $\text{M}^3\text{O}$ can also be applied to a more general scenario where we have missing entries in $M_o$ and multiple groups of data with distinct unknown correspondence. Experiments on  simulated data, the MovieLens 100K dataset and Yale B database show that $\text{M}^3\text{O}$ achieves state-of-the-art performance over several baselines and can recover the ground-truth correspondence with high accuracy.
\end{abstract}


\section{Introduction}
\label{sec:intro}
In the era of big data, one usually needs to utilize data gathered from multiple disparate platforms when accomplishing a specific task. However, the correspondence among the data samples from these different sources are often unknown due to either missing identity information or privacy reasons \citep{unnikrishnan2018unlabeled,gruteser2003privacy,das_sample--sample_2018}. Examples include the multi-image matching problem studied in \citep{ji_robust_2014,zeng_finding_2012,zhou_multi-image_2015}, the record linkage problem \citep{chan2001file} and the federated recommender system \citep{yang2020federated}.

In the simplest scenario, we have two data matrices $A=[a_1,...,a_n ]^\top$, $B=[b_1,...,b_n ]^\top$ with $a_i\in \mathbb{R}^{m_A}$ and $b_i\in \mathbb{R}^{m_B}$, which are from two different platforms (data sources). As discussed above, the  correspondence $(a_i,b_i)$ may not be available, and thereby the goal is to recover the underlying correspondence between $a_1,...,a_n$ and $b_{\tilde\pi{(1)}},...,b_{\tilde\pi{(n)}}$, where $\tilde\pi(\cdot)$ denotes an unknown permutation. We can translate such problem described above as a matrix recovery problem, i.e., to recover the matrix $M=[A,B]$ from the permuted observation $ M_o=[A,\tilde P B]$, where $\tilde P\in\Pc_n$ is an unknown permutation matrix and $\Pc_n $ denotes the set of all $n\times n$ permutation matrices. We term this problem as \textbf{M}atrix \textbf{R}ecovery with \textbf{U}nknown \textbf{C}orrespondence (\textbf{MRUC}).

Inspired by the classical low-rank model for matrix recovery \citep{Wright-Ma-2021,mazumder_spectral_2010,hastie_matrix_2015}, we especially focus on the scenario where the matrix $M$ features a certain low-rank structure. Such low-rank model has achieved great success in many applications like the recommender system \citep{schafer2007collaborative,mazumder_spectral_2010} and the image recovery and alignment problem \citep{zeng_finding_2012,zhou_multi-image_2015}. By denoting $B_o=\tilde PB$, we want to solve the following rank minimization problem for MRUC, 
\begin{align}
  \label{p:minimize_rank}
  \underset{ P\in \Pc_n}{\text{min }} \text{rank}([A, P B_o]).
\end{align} 

\textbf{Practical applications.}  It is known that the recommender system often suffers from  data sparsity \citep{zhang2012multi} because  users typically only provide ratings for very  few items. To enlarge the set of observable ratings for each user, we may harness extra data from multiple platforms (Netflix, Amazon, Youtube, etc.). One classical work on this problem is the multi-domain recommender system considered in \citep{zhang2012multi}. Unfortunately, their work neglects a crucial issue that data from these diverse platforms (or domains) are not always well aligned for two primary reasons. The first  is that the same user may use different identities, or even leave nothing about their identities, on these platforms. Another reason is that, those platforms are not allowed to share with each other the identity information about their users for preserving privacy. Another application is the visual permutation learning problem \citep{santa2017deeppermnet}, where one needs to recover the original image from the {\it shuffled} pixels. Both of the two applications give rise to a challenging extension of the MRUC problem, where we not only need to recover multiple correspondence across different data sources, but also face the difficulty of dealing with the missing values in data matrix .

\textbf{Relationship to the multivariate unlabeled sensing problem.} Problem
\eqref{p:minimize_rank} is closely related to the \textbf{M}ultivariate \textbf{U}nlabeled \textbf{S}ensing (\textbf{MUS}) problem, which has been studied in \citep{pananjady2017denoising,zhang2019benefits,zhang2019permutation,zhang2020optimal,slawski2020sparse,slawski2020two}. Specifically, the MUS is the multivariate linear {\it regression} problem with unknown correspondence, i.e., it solves \begin{align}
  \label{p:MUS}
  \min_{P\in\Pc_n,W\in \Rbb^{m_2\times m_1}} \|Y-PXW\|_F^2,
\end{align}

where $W\in\Rbb^{m_2\times m_1}$ is the regression coefficient matrix, $Y\in \Rbb^{n\times m_1}$ and $X\in \Rbb^{n\times m_2}$ denotes the output and the permuted input respectively, and  $\left\|\cdot\right\|_F$ is the matrix Frobenius norm. In fact, a concurrent work \citep{yao2021unlabeled} studies the same rank minimization problem as \eqref{p:minimize_rank}, but their approach is to solve it using the algorithm developed for MUS problem. Despite of the similarity to the MUS problem, we remark that MRUC problem  has it own distinct features and, as shown in Section \ref{sec:experiment},  the algorithm for the MUS algorithm can not be directly and effectively applied, especially when there are multiple unknown correspondence and missing entries to be considered. 

\textbf{Related works. }
To the best of our knowledge, the concurrent and independent \citep{yao2021unlabeled} is the only work that also considers the MRUC problem. Theoretically, \citep{yao2021unlabeled} showed that there exists a non-empty open subset $U\subseteq \Rbb^{n\times (m_1+m_2)}$, such that $\forall M\in U$, solving \eqref{p:minimize_rank} is bound to recover the original correspondence. However, they only proved the existence of such subset $U$ and did not provide a concrete characterization of it. Regarding the algorithm design, \citep{yao2021unlabeled} follows the idea of \citep{slawski2020sparse,slawski2020two} and treats  problem \eqref{p:minimize_rank} heuristically as a MUS problem. However, there are three main drawbacks in  their algorithm that largely limit its practical value. First, it can only work when the data is sparsely permuted, i.e., there are only $k$ data vectors being shuffled with $k\ll n$; Second, they only consider the scenario with a single unknown correspondence; Last but not least, their method can not deal with data with missing values.

\textbf{Contributions of this work. }
 Our contributions in this work lie in both theoretical and practical aspects. Theoretically, we are the first to rigorously study how the rank of the data matrix is perturbed by the permutation, and show that   problem \eqref{p:minimize_rank} can be used to recover a generic low-rank random matrix almost surely. Besides, we also propose a nuclear norm minimization problem as a surrogate for  problem \eqref{p:minimize_rank}. The most important theoretical result in this work is that we provide a non-asymptotic analysis to bound the error of the nuclear norm minimization problem under a mild assumption. Practically, we propose an efficient algorithm $\text{M}^3\text{O}$ that solves the nuclear norm minimization problem, which overcomes the aforementioned three shortcomings in \citep{yao2021unlabeled}. Notably, $\text{M}^3\text{O}$ works very well even for an extremely difficult task, where we need to recover multiple unknown correspondence from the data that are densely permuted and contain missing values. We remark that this is so far a challenging problem unexplored in the existing literature.

\textbf{Outline. }
For conciseness, we will first study the MRUC problem with single unknown correspondence, and then show that the theoretical results and the algorithm can be readily extended to the more complicated scenarios. We start with building the theoretical results for \eqref{p:minimize_rank} and its convex relaxation in Section \ref{sec:low-rank}. Then, the algorithm is developed in Section \ref{sec:algorithm}. The simulation results are presented in Section \ref{sec:experiment} and the conclusions are drawn in Section \ref{sec:discussion}.

  \textbf{Notations.} 
 Given two matrices $X,Y\in\mathbb{R}^{n\times m}$, we denote $\langle X,Y\rangle=\sum_{i=1}^{n} \sum_{j=1}^{m} X_{ij}Y_{ij}$ as the matrix inner product. We denote $X(i)$ as the $i$th row of the matrix $X$ and $X(i,j)$ as the element at  the $i$th row and the $j$th column. We denote $\mathbf{1}_m\in\mathbb{R}^m$  and $\mathbf{1}_{n\times m}\in\mathbb{R}^{n\times m}$ as the all-one vector and matrix, respectively, and $I_n$ be the $n\times n$ identity matrix. For $\alpha\in\mathbb{R}^m$, $\beta\in\mathbb{R}^n$, we define the operator $\oplus$ as $\alpha\oplus\beta=\alpha\mathbf{1}_n^\top+\mathbf{1}_m\beta^\top\in\mathbb{R}^{m\times n}$. We denote $\|\cdot\|_*$ as the nuclear norm for matrices. For vectors, we denote $\|\cdot\|_0$, $\|\cdot\|_1$ as the zero norm and 1-norm respectively.

 \section{Matrix Recovery via a Low-rank Model}
\label{sec:low-rank}

\textbf{How is the matrix rank  perturbed by the row permutation?}
To answer this fundamental question, we first introduce the cycle decomposition of a permutation. 

\begin{definition}[Cycle decomposition of a permutation \citep{dummit1991abstract}] Let $S$ be a finite set and $\pi(\cdot)$ be a permutation on $S$. A cycle $(a_1,...,a_n)$ is a permutation sending $a_j$ to $a_{j+1}$ for $1\leq j\leq n-1$ and $a_n$ to $a_1$. Then a cycle decomposition of $\pi(\cdot)$ is an expression of $\pi(\cdot)$ as a union of several disjoint cycles\footnote{Two cycles are disjoint if they do not have common elements}.
\end{definition}

It can be verified that any permutation on a finite set has a unique cycle decomposition \citep{dummit1991abstract}. Therefore, we can define the {\it cycle number} of a permutation $\pi(\cdot)$ as the number of disjoint cycles with length greater than 1, which is denoted as $C(\pi)$. We also define the non-sparsity of a permutation as the Hamming distance between it and the original sequence, i.e., $H(\pi)=\sum_{s\in S}\mathbb{I}[\pi(s)\neq s]$. It is obvious that $H(\pi)>C(\pi)$ if $\pi$ is not an identity permutation. As a simple example, we consider the permutation $\pi(\cdot)$ that maps the sequence (1,2,3,4,5,6) to (3,1,2,5,4,6). Now the cycle decomposition for it is $\pi(\cdot)=(132)(45)(6)$, and $C(\pi)=2$, $H(\pi)=5$. 

In all the following theoretical results, we denote the original matrix as $M={\bm[}A,B{\bm ]}\in \Rbb^{n\times m} $  with $A\in \Rbb^{n\times m_A}$, $B\in \Rbb^{n\times m_B}$, and ${\rm rank}(M)=r$, ${\rm rank}(A)=r_A$, ${\rm rank}(B)=r_B$. We denote the corresponding permutation as $\pi_P(\cdot)$ for any permutation matrix $P\in\Pc_n$. The following proposition says that the  perturbation effect of a permutation $\pi$ on the rank of $M$ becomes stronger, if  $\pi$ permutes more rows and contains less cycles.

\begin{proposition} \label{prop:MR} 
  $\forall P\in\Pc_n$, we have
  \begin{align}
      \label{eq:MR1}
      {\rm rank}({\bm [}A,PB{\bm]})\leq \min\{n,m,r_A+r_B,r+H(\pi_P)-C(\pi_P)\}.
  \end{align} 
\end{proposition}

We have similar result for the case with multiple permutation, which is summarized in  Corollary \ref{col:multiple_rank} in Appendix \ref{app:proof}. It turns out that, without any assumption on $M$, \eqref{eq:MR1} is the tightest upper bound for the rank of a perturbed matrix. Notably, the following proposition says that the upper bound in \eqref{eq:MR1} is attained with probability 1 for a generic low-rank random matrix.

\begin{definition} A probability distribution on $\Rbb$ is called a proper distribution if its density function $p(\cdot)$ is absolutely continuous with respect the Lebesgue measure on $\Rbb$.
\end{definition}
\begin{proposition}
  \label{prop:attained_rank}
  If the original matrix  $M$ is a random matrix with $M=RE$ where $R\in \Rbb^{n\times r}$ and $E\in\Rbb^{r\times m}$ are two random matrices whose entries are i.i.d and follow a proper distribution on $\Rbb$ , and  $r\leq \min\{\sqrt{\frac n 2},m_A,m_B\}$, then $\forall P\in\Pc_n$, the equality 
  \begin{align}
    \label{p:attained_rank}
      {\rm rank}({\bm [}A,PB{\bm]})= \min\{2r,r+H(\pi_P)-C(\pi_P)\}
  \end{align}
  holds  with probability 1.
\end{proposition} 

\textbf{Convex relaxation for the rank function.}
Despite the previous theoretical justification for problem \eqref{p:minimize_rank}, it is non-convex and non-smooth. Another crucial issue is that we often have a noisy observation matrix and it is well known that the rank function is extremely sensitive to the additive noise. In this paper, we assume that the observation matrix is corrupted by i.i.d Gaussian additive noise, i.e.,  $$M_o=[A_o,B_o]=[A,\tilde P B]+W,\text{ where }W(i,j)\sim \Nc(0,\sigma^2),$$

where $\sigma^2$ reflects the strength of the noise. We first denote the singular values of a matrix $X\in\Rbb^{n\times m}$ as $\sigma_X^1,...,\sigma_X^k$ where $k=\min \{n,m\}$. Since  $\text{rank}(X)=\|[\sigma_X^1,...,\sigma_X^k]\|_0$, from Proposition \ref{prop:attained_rank} we can view the perturbation effect  of a permutation to a low-rank matrix as breaking the sparsity of its singular values. This view leads naturally to the well-known 1-norm minimization problem which has been proven  robust to additive noise and can yield a sparse solution \citep{Wright-Ma-2021}, i.e., 
\begin{align}
  \label{p:nuclear_norm}
  \min_{P\in \Pc_n} \|[A_o,PB_o]\|_*= \|[\sigma_{M_o}^1,...,\sigma_{M_o}^k]\|_1.
\end{align}
Since for an arbitrary matrix, the 1-norm of its singular values is equivalent to its nuclear norm, we  refer problem \eqref{p:nuclear_norm} as the nuclear norm minimization problem.

\textbf{Theoretical justification for the nuclear norm. }
Nuclear norm has a long history  used as a convex surrogate for the rank, and it has been theoretically justified for applications like low-rank matrix completion \citep{candes2010power,Wright-Ma-2021}. It is also important to see whether the nuclear norm is still a good surrogate for the rank minimization problem \eqref{p:minimize_rank}. In this work, we establish a sufficient condition on $A$ and $B$ under which  problem \eqref{p:nuclear_norm} is provably justified for correspondence recovery.  We denote  $A=\sum_{i=1}^{r_A}\sigma_A^i u_A^i v_A^{i\top},\ B=\sum_{i=1}^{r_B}\sigma_B^i u_B^i v_B^{i\top}$ as the singular values decomposition of $A$ and $B$,  where the $\sigma_A^i$ and $\sigma_B^i$ are the non-zero singular values.

Firstly, from the definition of nuclear norm, it can be simply verified for any $P\in \Pc_n$ that

\begin{align}
  \label{p:ineq1}
  -Z/N \leq (\|[A,PB]\|_*-\|M\|_*)/\|M\|_*\leq Z/N,
\end{align}

where we denote  $N=\max\{\|A\|_*,\|B\|_*\}$ and $Z=\min\{\|A\|_*,\|B\|_*\}$. The inequality \eqref{p:ineq1} indicates that $A$ and $B$ should have comparable magnitude, i.e., $\|A\|_*\approx \|B\|_*$, otherwise the influence of the permutation will be less significant. Therefore, we are interested in the scenario where the singular values of $A$ and $B$ are comparable, which is described as the following Assumption \ref{asp:asp1}.
\begin{assumption} There exists a constant $\epsilon_1\geq 0$ such that
  \label{asp:asp1}
  \begin{align}
    \label{p:cond1}
  |\sigma_A^i-\sigma_B^i| \leq \epsilon_1,\ \forall i=1,..,r,
  \end{align} where we denote that $\sigma_A^i=0$ if $i>r_A$, and $\sigma_B^i=0$ if $i>r_B$.
\end{assumption}
  Similar to the matrix rank, we also need a proper low-rank assumption on the matrix $M$ for the nuclear norm. In this work, we particularly study the scenario that the left singular vectors of $A$ and $B$ are similar, which we formally describe as Assumption \ref{asp:asp2}. We  refer Assumption \ref{asp:asp2} as a proper low-rank assumption, because it indicates that  the column space of $M$ can be approximated by the column space of one of its submatrices. 
 \begin{assumption} There exists a constant $\epsilon_2\geq 0$ such that
  \label{asp:asp2}
  \begin{align}
    \label{p:cond2}
  \|u_A^i-u_B^i\|\leq \epsilon_2, \forall i=1,...,T,
  \end{align}
  where we denote $T=\min\{r_A,r_B\}$.
\end{assumption}

Furthermore, we also need that all the column singular vectors $u_A^1,...,u_A^T,u_B^1,...,u_B^T$ are variant under any $P\in\Pc_n$ with $P\neq I_n$: we define a vector $u\in \Rbb^n$ to be variant under a $P\in \Pc_n$ if $Pu\neq u$. One simple and weak condition for a vector $u$ to satisfy such property  is that $u$ dose not contains duplicated elements, which leads to the following Assumption \ref{asp:asp3}.
\begin{assumption} There exists a constant $\epsilon_3\geq 0$ such that
  \label{asp:asp3}
  \begin{align}
    \label{p:cond3}
    \min_{u\in U} \min_{i\neq j} |u(i)-u(j)|\geq \epsilon_3>0,  
  \end{align}where $U=\{u_A^1,...,u_A^T,u_B^1,...,u_B^T\}$.
\end{assumption}

In summary, the assumptions mentioned above feature a typical low-rank structure in $M$, and implies that the nuclear norm of $M$ is sensitive to permutation.
With the three assumptions, we have the following important theorem, which  provides high probability bound for the approximation error of \eqref{p:nuclear_norm}.

We denote the solution to \eqref{p:nuclear_norm} as $P^*$, and let $\pi^*$ and $\tilde \pi$ be the corresponding permutation to the  permutation matrices $P^{*\top}$ and $\tilde P$, respectively.  We define the difference between the two permutation $\pi^*$ and $\tilde \pi$ as the {\it Hamming} distance 
$$d_H(\pi^*,\tilde \pi)\stackrel{\text{def.}}=\sum_{i=1}^n\mathbb{I}(\pi^*(i)\neq \tilde \pi(i)).$$

\begin{theorem}
  \label{thm:error_bound}
  Under  Assumptions \ref{asp:asp1}, \ref{asp:asp2} and \ref{asp:asp3}, if additionally  $\epsilon_1\leq \frac{M}{4r}$, $\epsilon_2\leq \min\{\frac{1}{2\sqrt{2T}},\frac{\sqrt{2}M}{2N}\}$, and $\sigma \leq \frac{M}{16L^2}$, then the following bound for the Hamming distance
  \begin{align}
    \label{p:error_bound}
    d_{H}(\pi^*,\tilde \pi)\leq \frac{2}{\epsilon_3^2}\left(2-\left(\frac{\sqrt{2}D}{D+(\sqrt{2}+2)\epsilon_1r+\sqrt{2}\epsilon_2N+2\sqrt{2DL\sigma}} - \sqrt{T}\epsilon_2\right)^2 \right)
  \end{align} holds with probability at least $1-2\exp\{-\frac{D}{8L\sigma}\}$, where $L=\max\{n,m\}$, $D=\|A\|_*+\|B\|_*$.
\end{theorem}

   The proof to all the aforementioned theoretical results are provided in Appendix \ref{app:proof}. 
  
   \textbf{Remark 1. }From  Theorem \ref{thm:error_bound} we can see that when $\epsilon_3>0$, and $\epsilon_1\to 0$, $\epsilon_2\to 0$, $\ {\sigma}\to 0$, the error  $d_{H}(\pi^*,\tilde \pi)$ will converge to zero with probability 1. Furthermore, we can also discover that the correspondence can be difficult to recover when: 

   \begin{itemize}[leftmargin=*]
    \setlength\itemsep{0.01em}
     \item The rank of original matrix $M$ is high, which can be seen from \eqref{p:error_bound}.
     \item The magnitude  of $A$ and $B$ w.r.t rank or nuclear norm are not comparable, which can be seen from \eqref{p:ineq1} and \eqref{p:cond1}.
     \item The strength of noise is high, which can be seen from  \eqref{p:error_bound} and the probability in Theorem \ref{thm:error_bound}.
   \end{itemize}

   Notably, the numerical experiments in Section \ref{sec:experiment1} corroborate our claims as well.

  \begin{figure}[htbp]
    \centering 
      \includegraphics[width=7cm]{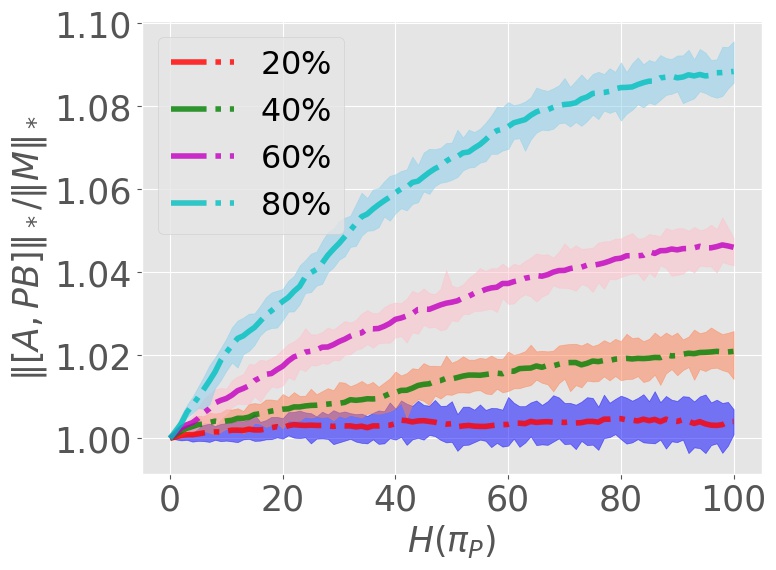}
      \caption{The relationship \eqref{p:relation} under different percentages of observable entries.}
      \label{fig:nuclearincrease}
    \end{figure}
  \textbf{Remark 2. }Additionally, from the proof of  Theorem \ref{thm:error_bound} we find that the fundamental reason for the success of \eqref{p:nuclear_norm} is that if $M$ satisfies the previous assumptions, we have  \begin{align}
      \label{p:relation}
      \|[A,P B]\|_*/\|M\|_*\approx O\left(\left(1-H(\pi_P)/2n\right)^{-\frac{1}{2}}\right).
    \end{align}

    In many applications, we can only observe part of the full data. Therefore, it is also worthwhile to investigate  whether \eqref{p:relation} still holds when we can only access a small subset of the entries in $M_o$. Notably, Figure \ref{fig:nuclearincrease} gives the positive answer and shows that the relationship \eqref{p:relation} is gracefully degraded when the percentage of observable entries is decreasing. This phenomenon is remarkable since it indicates the original correspondence can be recovered from only part of the full data. The matrices used to generate  Figure \ref{fig:nuclearincrease} are the same as those in  Section \ref{sec:experiment1}, and the nuclear norm is computed approximately by first filling  the missing entries using Soft-Impute algorithm \citep{mazumder_spectral_2010}.

\section{Algorithm}
\label{sec:algorithm}

In this section, we consider the scenario with missing values, i.e.,  our observed data is $\Pc_\Omega (M_o)=\Pc_\Omega([A_o,B_o])$, where $\Pc_\Omega$ is an operator that selects  entries that are in the set of observable indices  $\Omega$.
In this scenario,  problem \eqref{p:nuclear_norm} can not be directly used since the evaluation of the  nuclear norm and  optimization of the permutation are coupled together. Inspired by the matrix completion method \citep{hastie_matrix_2015,mazumder_spectral_2010}, we propose to solve an alternative form of \eqref{p:nuclear_norm} as follows, 

\begin{align}
  \label{p:unrelaxed}
  \min_{\widehat M\in \Rbb^{n\times m}} \min_{P\in \Pc_n}  &\left\|\Pc_\Omega([A_o, P B_o])-\Pc_\Omega(\widehat M)\right\|_F^2+\lambda\left\|\widehat M\right\|_*,
\end{align} 

where $\lambda>0$ is the penalty coefficient. We denote that $\widehat M=[\widehat M_A,\widehat M_B]$ and $\widehat M_A,\widehat M_B$  are the two submatrices  with the same dimension as $A_o$ and $B_o$ respectively. We can write \eqref{p:unrelaxed} equivalently as

\begin{align}
  \label{p:rewrite}
  \min_{\widehat M\in \Rbb^{n\times m}}\min_{P\in \Pc_n} \left\|\Pc_\Omega(A_o) -\Pc_\Omega(\widehat M_A)\right\|_F^2+\langle C(\widehat M_B), P \rangle+\lambda\left\|\widehat M\right\|_*,
\end{align}

where  $C(\widehat M_B)\in\mathbb{R}^{n\times n}$ is the pairing cost matrix with

\begin{align*}
  C(\widehat M_B)(i,j)=\sum_{(j,j'')\in \Omega}\bigg(\widehat M_B(i,j'')-B_o(j,j'')\bigg)^2,~\forall i,j=1,...,n.
\end{align*}

\textbf{Baseline algorithm.}
A conventional  strategy to handle an optimization problem like \eqref{p:rewrite} is the alternating minimization or the block coordinate descent algorithm \citep{abid2017linear}. Specifically, it executes the following two updates iteratively until converge.
 
\begin{align}
  \label{p:BCD1}
  &\widehat M^{\text{new}}\leftarrow \underset{\widehat M\in \Rbb^{n\times m}}{\arg\min}\left\|\Pc_\Omega([A_o,\widehat P^{\text{old}} B_o])-\Pc_\Omega(\widehat M)\right\|_F^2+\lambda\left\|\widehat M\right\|_*,\\
  \label{p:BCD2}
  &\widehat P^{\text{new}}\leftarrow \underset{P\in \Pc_n}{\arg\min}\  \langle C(\widehat M^{\text{new}}_B),P\rangle.
\end{align}

The first update step \eqref{p:BCD1} is a convex optimization problem and can be solved by the proximal gradient algorithm \citep{mazumder_spectral_2010}. The second update step \eqref{p:BCD2} is actually a discrete optimal transport problem which can be solved by the classical Hungarian algorithm with time complexity $O(n^3)$ \citep{jonker1986improving}. However, as we will see in the Section \ref{sec:experiment}, this algorithm performs poorly, and it is  likely to fall into an undesirable local solution quickly in practice. Specifically, the main reason is that the solution of \eqref{p:BCD2} is often not unique and  a small change in $\widehat M_B$ would lead to large change of $\widehat P$.  To address this issue, we propose a novel and efficient algorithm $\text{M}^3\text{O}$ algorithm based on the entropic optimal transport \citep{peyre2019computational} and min-max optimization \citep{jin2020local}.

\textbf{Smoothing the permutation with entropy regularization.}
For any $a\in\mathbb{R}^n,b\in\mathbb{R}^m$, we define  \begin{align*}
  \Pi(a,b)=\{S\in \mathbb{R}^{n\times m}:S\mathbf{1}_m=a,S^\top \mathbf{1}_n=b,S(i,j)\geq 0,~\forall i,j\},
\end{align*} 

which is also known as the Birkhoff polytope. The famous Birkhoff-von Neumann theorem \citep{birkhoff1946three} states that the set of extremal points of $\Pi(\mathbf{1}_n,\mathbf{1}_n)$ is equal to $\Pc_n$. Inspired by \citep{xie_hypergradient_2020} and the interior point method for linear programming \citep{bertsekas1997nonlinear}, in order to smooth the optimization process of the baseline algorithm, we relax $P$ from being an exact permutation matrix, i.e., to keep $P$ staying inside the Birkhoff polytope $\Pi(\mathbf{1}_n,\mathbf{1}_n)$. That is, we propose to replace the combinatorial problem \eqref{p:BCD2} with the following continuous optimization problem 

\begin{align}
  \label{p:entropyOT}
  \underset{P\in \Pi(\mathbf{1}_n,\mathbf{1}_n)}{\text{min }} &\langle C(\widehat M_B),P\rangle+\epsilon H(P),
\end{align}

where  $H(P)\stackrel{\text{def.}}=\sum_{i,j}P(i,j)(\log(P(i,j))-1)$ is the matrix negative entropy and $\epsilon>0$ is the regularization coefficient. Notably, \eqref{p:entropyOT} is also known as the Entropic Optimal Transport (EOT) problem \citep{peyre2019computational}, which is a strongly convex optimization problem and can be solved  roughly in the $O(n^2)$ complexity by the Sinkhorn algorithm. Specifically, the Sinkhorn algorithm solves the dual problem of \eqref{p:entropyOT},   \begin{align}\label{EOT dual}
	\max _{\alpha,\beta \in \mathbb{R}^{n}}  W_\epsilon(\widehat M_B,\alpha,\beta)\stackrel{\text{def.}}=\left\langle\mathbf{1}_{n}, \alpha\right\rangle+\left\langle\mathbf{1}_{n}, \beta\right\rangle-\epsilon\bigg\langle  \mathbf{1}_{n\times n} , \text{exp}\bigg\{\frac{\alpha\oplus\beta-C(\widehat M_B)}{\epsilon}\bigg\}\bigg\rangle,
\end{align} 

which reduces the variables dimension from $n^2$ to $2n$ and is thus greatly favorable in the high dimension scenario. By substituting the inner minimization problem of \eqref{p:rewrite} with \eqref{p:entropyOT}, we end up with solving the following  unconstrained min-max optimization problem 

\begin{align}
  \label{p:entropydual}
  \underset{\widehat M}{\min}\underset{\alpha,\beta}{\ \max} \left\|A-\widehat M_A\right\|_F^2+W_\epsilon(\widehat M_B,\alpha,\beta)+\lambda\left\|\widehat M\right\|_*&.
\end{align}

Follows the idea of \citep{jin2020local}, we consider to adopt a proximal gradient algorithm with a Max-Oracle for \eqref{p:entropydual}.
Specifically,  we employ the Skinhorn algorithm \citep{peyre2019computational} as the  Max-Oracle to retrieve an $\varepsilon$-good solution of the inner max problem \eqref{EOT dual}.
We summarize our proposed algorithm  $\text{M}^3\text{O}$ (\textbf{M}atrix recovery via \textbf{M}in-\textbf{M}ax \textbf{O}ptimization) in Algorithm \ref{alg:McubicO}, where $\text{prox}_{\lambda\left\|\cdot\right\|_*}(\cdot)$ is the proximal operator of nuclear norm and $\rho_k$ is the gradient stepsize. The convergence property of $\text{M}^3\text{O}$ can be obtained by following \citep{jin2020local}, which shows that, with a decaying stepsize, $\text{M}^3\text{O}$ is bound to converge to an $\varepsilon$-good Nash equilibrium within $O(\varepsilon^{-2})$ iterations.

\begin{algorithm}
  \caption{$\text{M}^3\text{O}$}\label{alg:McubicO}
  \While{\rm not converged}{
    For the tolerance $\varepsilon$, run the Sinkhorn algorithm to find $\alpha^*$, $\beta^*$ such that $$W_\epsilon(\widehat M_B^{k},\alpha^*,\beta^*)>\max_{\alpha,\beta}\ W_\epsilon(\widehat M_B^{k},\alpha,\beta)-\varepsilon;$$

  Perform
  $\widehat M^{\text{k+1}} \leftarrow \text{prox}_{\lambda\left\|\cdot\right\|_*} (\widehat M^{k}-\rho_k\nabla_{\widehat M} F_\epsilon(\widehat M^{k},\alpha^*,\beta^*)),$ where $$F_\epsilon(\widehat M,\alpha,\beta)\stackrel{\text{def.}}=\left\|A-\widehat M_A\right\|_F^2+W_\epsilon(\widehat M_B,\alpha,\beta);$$
   }
  \end{algorithm}

\textbf{Remark 3. }A recent work \citep{pmlr-v115-xie20b} proposes a decaying strategy for the entropy regularization coefficient $\epsilon$ in \eqref{p:entropyOT} so that the optimal solutions of \eqref{p:BCD2} and \eqref{p:entropyOT} do not deviate too much. Inspired by it, in our practice, we take large $\epsilon$ in the beginning and gradually shrink it by half until the objective function stops improving for $K$ steps.

\textbf{Remark 4. }A useful trick is that we should not take large stepsize $\rho_k$ in the early  iterations because the permutation matrix could still be far away from the optimal one. However,  a small stepsize would lead to slow convergence. Heuristically, we propose an adaptive stepsize strategy that performs  well in practice. For the   solution of (\ref{p:entropyOT}) $\widehat P_k$ at the $k$th iteration, we  compute the two statistics  $$\delta_k=\left\|\widehat P_{k-1}-\widehat P_{k}\right\|_F^2/2n \text{ and }
  c_k=\left\|\text{max}_j \widehat P_k(\cdot,j)-\mathbf{1}_n\right\|_1/n.$$ 

  Here $\delta_k$ represents how fast the permutation matrix $\widehat P_k$ changes over the iterations, while $c_k$ measures how far the current $\widehat P_k$ is  close to an exact permutation matrix. Both $\delta_k$ and $c_k$  reflect the confidence on the current found correspondence.  Based on them,  we set the stepsize as $\rho_{k+1}=(1-\delta_k)(1-c_k)^\omega,$  where $\omega>0$ is a tunable parameter which is often set to a value between 0.5 to 3. $\omega$ actually trades off the convergence speed and final performance. The smaller the $\omega$, the faster the convergence. Therefore, a practical way is to start with a small $\omega$, and gradually increase it until the final performance stops improving.

\textbf{Remark 5. }As discussed in Section \ref{sec:intro}, in many cases we have to deal with the problem that involves multiple correspondence, i.e., we need to recover the matrix $M=[A,B_1,...,B_d]$ from the observation data $\Pc_\Omega(M_o)$, where 

$$M_o=[A_o,B_o^1,...,B_o^d]=[A,\tilde P_1B_1,...,\tilde P_dB_d]+W,$$

where $\tilde P_l\in\Pc_n$ and $W$ is a noise matrix. We refer such problem as the $\bm{d}$\textbf{-correspondence} problem. An important observation is that, although the number of possible correspondence increase exponentially as $d$ grows, the complexity of   M$^3$O per iteration only linearly increases with $d$ and can be implemented in a fully parallel fashion. Specifically, in this scenario, we solve the  problem

\begin{align}
  \label{p:multiple}
  \min_{\widehat M}\min_{P_1,...,P_d}&\left\|\Pc_{\Omega}(A_o) - \Pc_{\Omega}(\widehat M_A)\right\|_F^2+\sum_{l=1}^{d}\bigg\{\langle C(\widehat M_{B_l}), P_l \rangle+\epsilon H(P_l)\bigg\}+\lambda\left\|\widehat M\right\|_*,\\
  &\text{s.t. }P_l \in \Pi(\mathbf{1}_n,\mathbf{1}_n),\ l=1,...,d,\notag
\end{align}

where we denote $\widehat M=[\widehat M_A,\widehat M_{B_1},...,\widehat M_{B_d}]$. Here $\widehat M_A$ and $\widehat M_{B_l}$ have the same dimension with $A_o$ and $ B_o^l$, respectively. 
  One can find that the inner problems for solving $P_l$ are actually decoupled for each $l$, which guarantees  an efficient parallel implementation.

\textbf{Remark 6. }Since x problem \eqref{p:unrelaxed} has a similar form to that considered in \citep{mazumder_spectral_2010}. We adopt the same tuning strategy of $\lambda$ as in \citep{mazumder_spectral_2010}, which suggests that we should start with large $\lambda$ and gradually decrease it.

We relegate more details about $\text{M}^3\text{O}$ to Appendix \ref{app:algo}.

\begin{figure}[htbp]
	\centering
  \subfigure[Objective value]{
    \label{fig:exp1_Obj}
  \includegraphics[width=6cm]{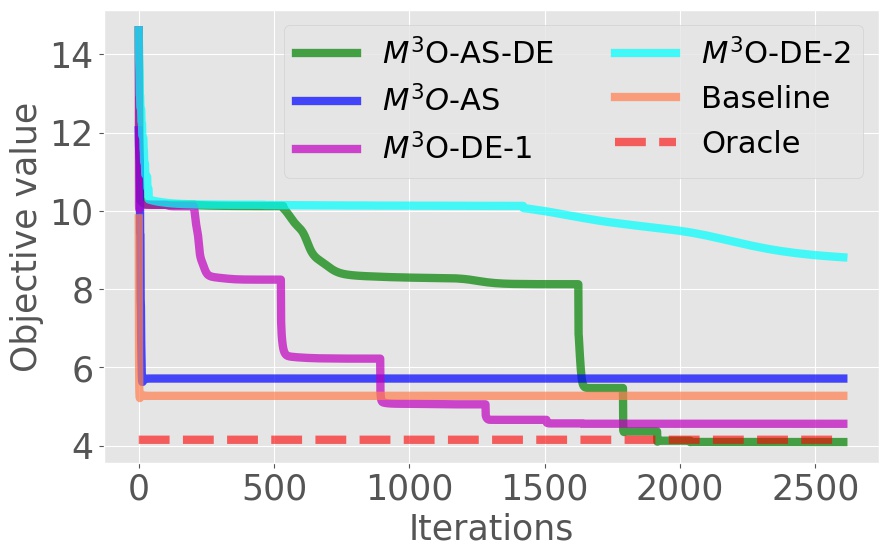}
  }
  \subfigure[Permutation error]{
    \label{fig:exp1_Perr}
  \includegraphics[width=6cm]{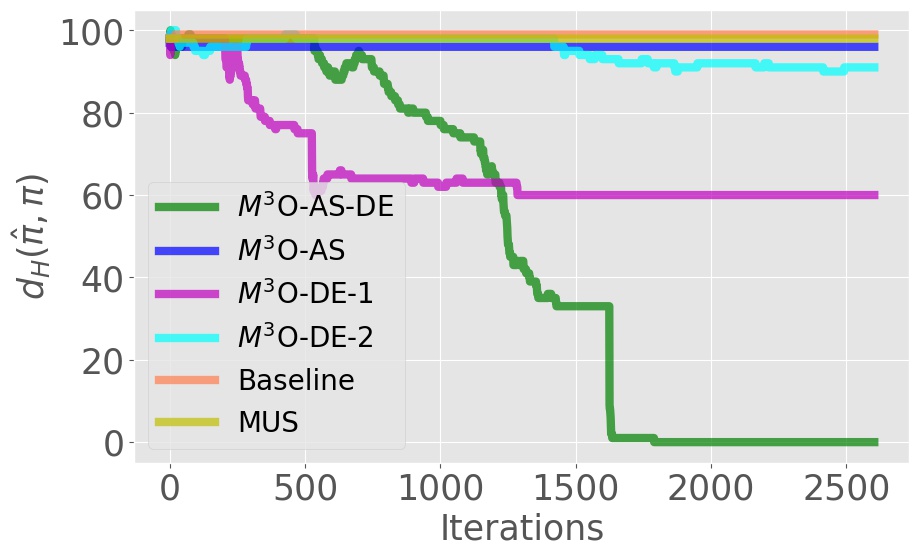}
  }
  \caption{ Performance of various algorithms on a simulated 1-correspondence problem.}\label{fig:exp1}
  \label{fig:multi_algo}
  \end{figure}

\section{Experiments}
\label{sec:experiment}

In this section, we evaluate our proposed $\text{M}^3\text{O}$ on both synthetic and real-world datasets, including the MovieLens 100K and the Extended Yale B dataset. We also provide an ablation study for the decaying entropy regularization strategy and the  adaptive stepsize strategy proposed in Remarks 3 and 4. In all the experiments, we employ the Soft-Impute algorithm \citep{mazumder_spectral_2010} as a standard algorithm for matrix completion.   Extra experiment details and auxiliary results can be
found in Appendix \ref{app:exp}.

\textbf{Algorithms.} We denote the following algorithms for comparison in all the experiments:

\begin{enumerate}[style=sameline,itemindent=0em,leftmargin=20pt]
   \item  \textit{Oracle}:  Running the Soft-Impute algorithm  with ground-truth correspondence.
  \item \textit{Baseline}: The Baseline algorithm in \eqref{p:BCD1} and \eqref{p:BCD2}.
  \item \textit{MUS}: Since there is currently no existing algorithm directly applicable to the scenario considered by \eqref{p:multiple}, inspired by \citep{yao2021unlabeled}, we modify and extend the algorithm in \citep{zhang2020optimal}, which is originally proposed for the MUS problem, to deal with the MRUC problem. The details of the adapted algorithm are provided in Appendix \ref{app:adapt}.
\end{enumerate}

\subsection{Synthetic data}
\label{sec:experiment1}

We first investigate the property of our proposed $\text{M}^3\text{O}$  algorithm on the synthetic data.

\textbf{Data generation.}
We generate the original data matrix in this form $M=R E+\eta W, $ where $R\in\mathbb{R}^{n\times r}$, $E \in \mathbb{R}^{r\times m}$, $W\in \mathbb{R}^{n\times m}$ and $\eta >0$ indicates the strength of the additive noise. The entries of $R$, $E$, $W$ are all i.i.d sampled from the  $\Nc (0,1)$. Then we split the data matrix $M$ by  $M=[A,B_1,...,B_d]$ where we denote $A\in \mathbb{R}^{n\times m_A}$, $B_1\in \mathbb{R}^{n\times m_1}$, ..., $B_d\in \mathbb{R}^{n\times m_d}$ to represent data from $d+1$ data sources. The permuted observation matrix $M_o$ is obtained by first generating $d$  permutation matrices $P_1,...,P_d$ randomly and independently, and then computing $ M_o=[A, P_1  B_1,..., P_d  B_d]$. Finally, we remove $(1-|\Omega|\cdot 100\%/(n\cdot m))$ percent of the entries of $M_o$ randomly and uniformly, where  $|\Omega|$ indicating the number of observable entries.

  \textbf{Ablation study.} We denote the following variants of $\text{M}^3\text{O}$ for the ablation study.

  \begin{enumerate}[style=sameline,itemindent=0em,leftmargin=20pt]
    {\item \textit{$\text{M}^3\text{O}$-AS-DE}:  $\text{M}^3\text{O}$  with both {A}dpative {S}tepsize and {D}ecaying {E}ntropy regularization. }
    {\item \textit{$\text{M}^3\text{O}$-DE}: $\text{M}^3\text{O}$ with Decaying Entropy regularization only. $\text{M}^3\text{O}$-DE-1 and $\text{M}^3\text{O}$-DE-2 adopt constant stepsize $\rho_k=0.5$ and $\rho_k=0.01$, respectively.}
    {\item \textit{$\text{M}^3\text{O}$-AS}: $\text{M}^3\text{O}$  with {A}dpative {S}tepsize only. The entropy coefficient $\epsilon$ is fixed to 0.0005.} 
  \end{enumerate}

  In the following results, we denote $\pi_l$ as the corresponding permutation to $ P_l$.  We initialize $\widehat M$ from Gaussian distribution for the $\text{M}^3\text{O}$ algorithm and its variants. We choose initial $\epsilon$ as 0.1 and $K=100$ as the default for the decaying entropy regularization, and set $\omega=3$ as the default for the adaptive stepsize. We also report  the achieved objective values of \eqref{p:multiple} for the tested algorithms, except for the MUS algorithm since it has a different objective. We denote $\hat \pi$ as the recovered permutation.

  \textbf{Results.}   Figure \ref{fig:multi_algo} displays the result under the setting  $\eta=0.1$, $|\Omega|\cdot 100\%/(n\cdot m)=80\%$, $n=m=100$, $r=5$, $d=1$, $m_A=60$ and $m_1=40$.  The algorithm {M}$^3$O-AS-DE achieves the best result, and can recover the ground-truth correspondence. {M}$^3$O-AS behaves similarly to Baseline and MUS. They all converge to a poor local solution quickly. {M}$^3$O-DE-1 converges quickly and also falls into a poor local solution due to large stepsize, while  {M}$^3$O-DE-2 adopts a small stepsize and hence suffers from slow convergence. Due to the superiority of {M}$^3$O-AS-DE over the other variants, in the following results, we refer {M}$^3$O as {M}$^3$O-AS-DE for short.
  
  Figure \ref{fig:various_setting} examine {M}$^3$O on a 1-correspondence problem under different regimes w.r.t $|\Omega|$, $\eta$, $r$ and $m_A/n$.  Here we use  $m_A/n$ to  control the difference of the magnitude of the submatrices. As we can see, the results are well aligned with our prediction in Remarks 1 and  2.

  Finally, we examine {M}$^3$O on a few d-correspondence problems. See Table \ref{tb:multi-perm} for various results, where we set $r=5$ and $\varepsilon=0.1$. Notice that for the 4-correspondence problem in the table, there are $(100!)^4$ possible correspondence. Even for such a difficult problem, {M}$^3$O is able to recover 61.5\% of the ground-truth correspondence with a good initialization. 

    \begin{figure}[htbp]
      \centering
      \subfigure[ $d_H$ v.s. $|\Omega|$]{
        \label{fig:exp2_Obj}
      \includegraphics[width=3.5cm]{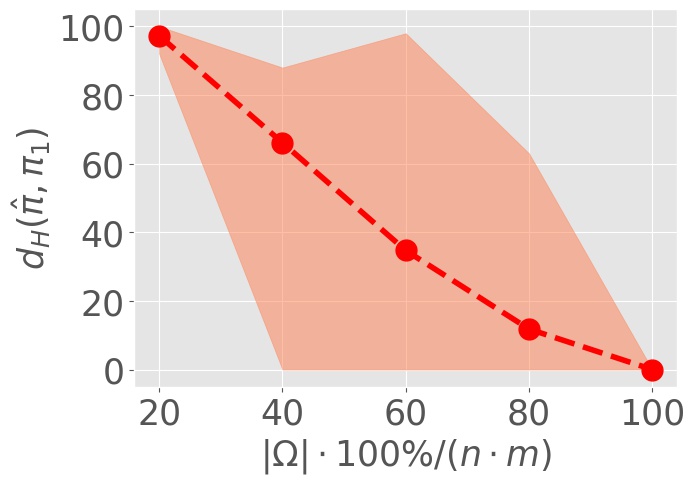}
      }
      \subfigure[ $d_H$ v.s. $\eta$]{
        \label{fig:exp3_Perr}
      \includegraphics[width=3.5cm]{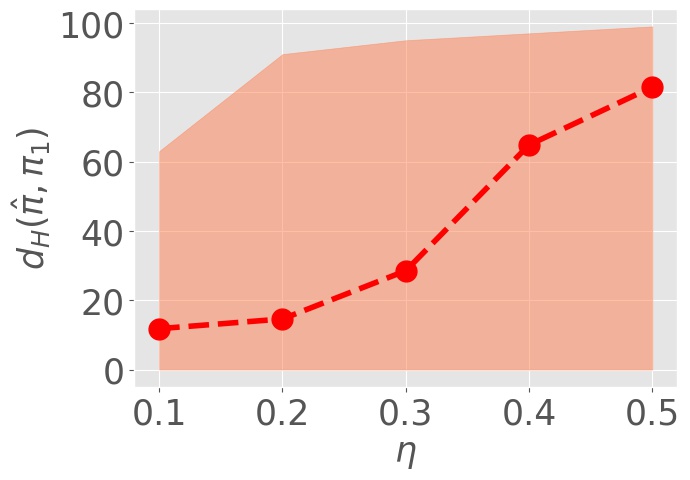}
      }
      \quad
      \centering
      \subfigure[ $d_H$ v.s. $r$]{
        \label{fig:exp3_r}
      \includegraphics[width=3.5cm]{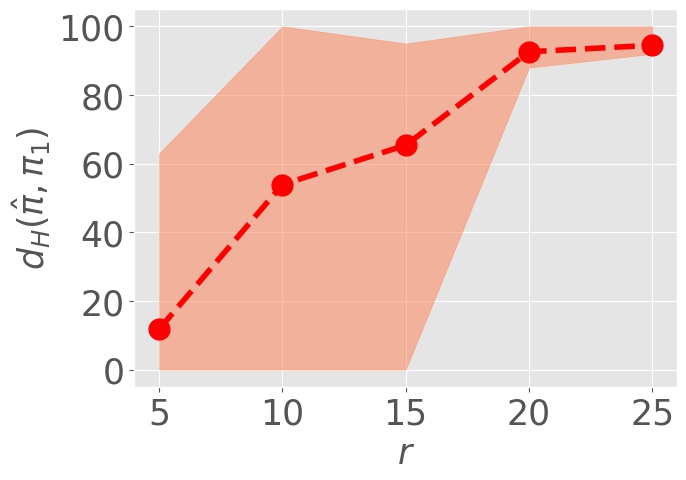}
      }
      \quad
      \subfigure[ $d_H$ v.s. $m_A/n$]{
        \label{fig:exp3_ma}
      \includegraphics[width=3.5cm]{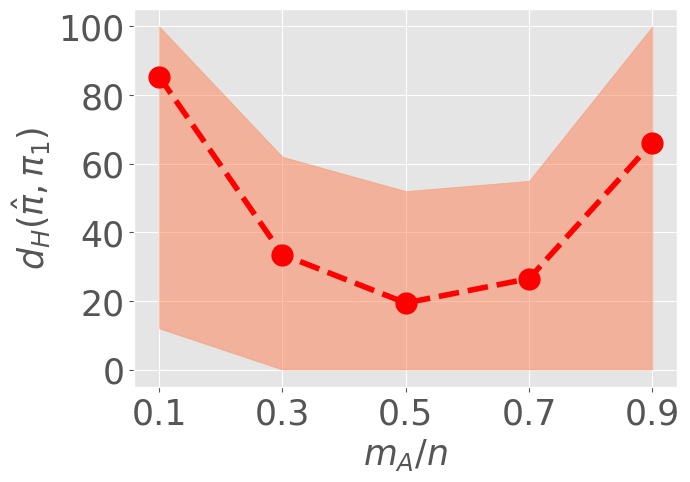}
      }
      \caption{ Performance of {{M}$^3$O} on a 1-correspondence problem under different levels of $|\Omega|$, $\eta$, $r$ and $m_A/n$. The default setting is $|\Omega|\cdot 100\%/(n\cdot m)=80\%$, $\eta=0.1$, $n=m=100$, $r=5$, $m_A=60$,  and $m_1=40$. The mean  with minimum and maximum are calculated from 10 different random initializations.}
      \label{fig:various_setting}
      \end{figure}

      \begin{table}[htbp]
        \centering

        \caption{ Performance of {M}$^3$O for various d-correspondence problems. The normalized permutation error $\sum_{l=1}^d d_H(\hat \pi_l,\pi_l)/d$  is reported as mean$\pm$std (min) over 10 different random initializations.}
        \label{tb:multi-perm}
     
        \begin{tabular}{ccccc}
        \toprule
        $(n,m_A,m_1,...,m_d)$  & $d$ & $|\Omega|\cdot 100\%/(n\cdot m)$  & $\sum_{l=1}^d d_H(\hat \pi_l,\pi_l)/d$ \\
        \midrule
        (100,40,30,30)& 2 & 40\% & $33.35\pm32.85$ (0.00)\\
        (100,20,40,40)& 2 & 40\% & $58.90\pm27.21$ (2.00)\\
        (100,45,25,25,25)& 3 & 50\% & $61.97\pm15.41$ (37.33)\\
        (100,40,25,25,25,25)& 4 & 60\% & $59.90\pm13.64$ (38.50)\\
        \bottomrule
        \end{tabular}
        \end{table}

        \subsection{Multi-domain recommender system without correspondence}

In this section, we study the performance of $\text{M}^3\text{O}$ on a real world dataset MovieLens 100K\footnote{https://grouplens.org/datasets/movielens/100k/}, which is a widely used movie recommendation dataset \citep{harper2015movielens}. In this application, we mainly focus on the metric Root Mean Squared Error (RMSE), i.e.,

 $$\text{RMSE}\stackrel{\text{def.}}=\sqrt{\frac1N\sum_{i,j}(\widehat M_{ij}-M_{ij})^2}.$$

\textbf{Data.} MovieLens 100K contains 100,000 ratings within the scale 1-5. The ratings are given by 943 users on 1,682 movies. Genre information about movies is also provided. We adopt a similar setting with \citep{zhang2012multi}. We extract five most popular genres, which are Comedy, Romance,  Drama,  Action,  Thriller respectively, to define the data from 5 different domains (or platforms). In addition to \citep{zhang2012multi}, we randomly permute the indexes of the users from these five domains respectively, so that the correspondence among these data become unknown. In this way, the problem belongs to the 4-correspondence problem as discussed before. The ratings are split randomly, with 80\% of them as the training data and the other 20\% of them as the test data.

\textbf{Algorithms.} We consider the following additional algorithms for comparison.

\begin{enumerate}[style=sameline,itemindent=0em,leftmargin=20pt]
  {\item \textit{SIC}: Running the Soft-Impute algorithm independently for the 5 different platforms. }
  {\item \textit{SIR}: Running the Soft-Impute algorithm  with Randomly generated correspondence.}
\end{enumerate}

\textbf{Results.}
As discussed in experiments on the simulated data, the exact recovery of correspondence becomes impossible due to the small amount of observable entries. Therefore, in the following experiment, since exact correspondence is not needed, we fix $\epsilon=0.05$ for $\text{M}^3\text{O}$. Table \ref{tb:federated} shows the results by averaging the RMSE on the test data over 10 different random seeds.

We can first see that the matrix completion with a wrong correspondence, i.e., SIR, can be harmful to the overall performance since it is even worse than the results of SIC. Notably,  although the ground-truth correspondence can not be recovered, each platform  can still benefit from $\text{M}^3\text{O}$ since it improves the performance over SIC. This is mainly because $\text{M}^3\text{O}$ is still able to correspond similar users for inferring missing ratings. On the contrary, since both Baseline and MUS can only establish an exact one-to-one correspondence for each user, they fail to improve SIC significantly. Remarkably, $\text{M}^3\text{O}$ is only inferior to the Oracle method a little, and even achieves lower test RMSE than the Oracle method on the Comedy genre. 

  \begin{table}[htbp]
 
    \centering
    \caption{  Test RMSE of various algorithms on MovieLens 100K}\label{tb:federated}

    \begin{tabular}{ccccccc}
    \toprule
    Method & Comedy & Romance & Drama & Action & Thriller & Total\\
    \midrule
    SIR&1.0202  &  1.0158 &   0.9808  &  0.9803  &  0.9811  &  0.9944\\
    
    SIC&0.9694  &  0.9695 &   0.9317 &   0.9175  &  0.9253 &   0.9418\\
    MUS&0.9659  &  0.9842 &   0.9423  &  0.9305  &  0.9306  &  0.9485\\
      Baseline &0.9728& 0.9562& 0.9379& 0.9105  & 0.9145&
          0.9395 \\
    $\text{M}^3\text{O}$& \textbf{0.9389} & \textbf{0.8787} & \textbf{0.9139} & \textbf{0.8556} & \textbf{0.8567} & \textbf{0.8948}\\
    Oracle& 0.9444   & 0.7825 &   0.9058 &   0.8176  &  0.8098 &   0.8667\\
    \bottomrule
    \end{tabular}
  
    \end{table}

  \subsection{Visual permutation recovery}

  \begin{figure}[htbp]
 
    \centering
  \subfigure[ Original]{
    \label{fig:real_face}
  \includegraphics[width=3.5cm]{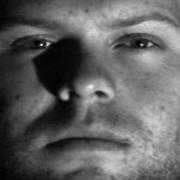}
  }
  \quad
  \subfigure[Corrupted]{
    \label{fig:perm_face}
  \includegraphics[width=3.5cm]{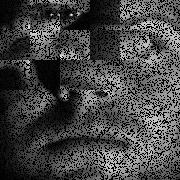}
  }
  \quad
  \centering
  \subfigure[Baseline]{
    \label{fig:base_face}
  \includegraphics[width=3.5cm]{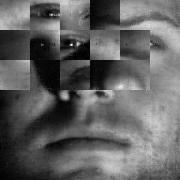}
  }
  \quad
  \subfigure[$\text{M}^3\text{O}$]{
    \label{fig:Mcubic_face}
  \includegraphics[width=3.5cm]{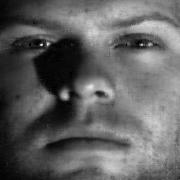}
  }

  \caption{Performance of {{M}$^3$O} on a face recovery problem.}
  \label{fig:face}
    \end{figure}
  We show that  M$^3$O  is flexible and can also be used to recover matrix that is not in the form $[A,PB]$. We can see this from the problem formulation in \eqref{p:rewrite}, where the cost matrix $C(\cdot)$ can be constructed in other ways as long as it is a function of a permutation. Typically, M$^3$O  can be used to solve a challenging face image recovery problem. The original face image with size $180\times 180$ in Figure \ref{fig:real_face} comes from the Extend Yale B database \citep{GeBeKr01}. The corrupted image is visualized in Figure \ref{fig:perm_face}, where the  pixel blocks with size $30\times 30$ in the upper left are shuffled randomly, and $30\%$ of the total pixels are removed. 
  This kind of problem is recently considered in \citep{santa2017deeppermnet}, which proposes to recover the corrupted image in a data-driven way using convolutional neural networks. However, we show that it is possible to recover the image without additional data by merely exploiting the underlying low-rank structure of the image itself.

  This experiment setting is similar to that in  \citep{yao2021unlabeled} but the algorithm in \citep{yao2021unlabeled} can not be applied since it can not work with the missing values. The  MUS algorithm is also not applicable since this problem can not be written in the form of linear regression problem. From Figure \ref{fig:base_face} and \ref{fig:Mcubic_face} we can find that $\text{M}^3\text{O}$ performs better than the Baseline, and can even recover the original orders of pixel blocks. More results similar to the Figure \ref{fig:face} and experiment details are provided in Appendix \ref{app:exp}.

\section{Conclusion}
\label{sec:discussion}

In this paper, we have studied the important MRUC problem where part of the observed submatrix is shuffled. This problem has not been well explored in the existing literature. Theoretically, we are the first to rigorously analyze the role of low-rank model in the MRUC problem, and is also the first to show that minimizing nuclear norm is provably efficient for recovering a typical low-rank matrix. For practical implementations, we propose a highly efficient algorithm, the {$\text{M}^3\text{O}$} algorithm, which is shown to consistently achieve the best performance over several baselines in all the tested scenarios.

It is worthwhile to point out that apart from the two applications we have studied in this paper, this problem could arise in more scenarios like the gnome assembly problem \citep{huang1999cap3}, the video pose tracking problem \citep{ganapathi2012real}  and the privacy-aware sensor networks \citep{gruteser2003privacy}, etc. We believes that our work provides a general framework to deal with unknown correspondence issue in these scenarios.

As we have shown in Figure \ref{fig:various_setting}, one major limit of our algorithm is the sensitivity to the initialization. The phenomenon is exacerbated when the additive noise is high or the numbers of observable entries are small. We suggest to try with a few different initialization strategy when applying {$\text{M}^3\text{O}$} to a specific task. Finding stable initialization strategy is also an important task for our future works.

\bibliographystyle{unsrtnat}
\bibliography{references}
\nocite{*} 






\appendix
\section{Appendix}

\input{supplement.tex}

\end{document}

%% file: source/arxiv_source/command.tex
\usepackage[utf8]{inputenc} 
\usepackage[T1]{fontenc}    
\usepackage{hyperref}       
\usepackage{url}            
\usepackage{booktabs}       
\usepackage{nicefrac}       
\usepackage{microtype}      
\usepackage{lipsum}		
\usepackage{graphicx}
\usepackage{natbib}
\usepackage{doi}
\usepackage{xcolor}         
\usepackage{amsmath,amsfonts,amssymb,amsthm,bm}  
\usepackage[ruled,linesnumbered]{algorithm2e}
\usepackage{wrapfig}
\usepackage{subfigure}
\usepackage{caption}
\usepackage{enumitem}

\newtheorem{assumption}{Assumption}
\newtheorem{proposition}{Proposition}
\newtheorem{theorem}{Theorem}
\newtheorem{lemma}{Lemma}

\newtheorem{definition}{Definition}
\newtheorem{corollary}{Corollary}

\input{../math_command.tex}

%% file: source/math_command.tex
\newcommand{\hide}[1]{} 

\newcommand\Rbb{\ensuremath{{\mathbb{R}}}}

\newcommand\Sc{\ensuremath{\mathcal{S}}}

\newcommand\Pc{\ensuremath{{\mathcal{P}}}}

\newcommand\Nc{\ensuremath{{\mathcal{N}}}}

%% file: supplement.tex
\subsection{Proof for the Theoretical Results}
\label{app:proof}
\begin{proof}[Proof of Proposition \ref{prop:MR}]
  We denote that $a_1,..,a_{r_A}$ as the linear bases of the column space of $A$. We can extend them to the bases of the column space of $M$ as $a_1,..,a_{r_A},b_1,...,b_{r-r_A}$. In this way, there must exists a matrix $Q\in \Rbb^{r\times m_B}$ such that $$B={\bm[}a_1,..,a_{r_A},b_1,...,b_{r-r_A}{\bm ]}Q.$$ Hence, we have $$PB={\bm[}Pa_1,..,Pa_{r_A},Pb_1,...,Pb_{r-r_A}{\bm ]}Q.$$
    Similarly, there must exists a matrix $T\in \Rbb^{r_A\times m_A}$ such that $$A={\bm[}a_1,..,a_{r_A}{\bm ]}T.$$
    Hence, we obtain that $${\bm [}A,PB{\bm]}={\bm[}a_1,..,a_{r_A},Pa_1,..,Pa_{r_A},Pb_1,...,Pb_{r-r_A}{\bm ]}\begin{bmatrix}
        T &  0 \\
        0 & Q \\
        \end{bmatrix}.$$

        Now, we have
        \begin{align}
            \label{eq:MR2}
            &{\rm rank}({\bm [}A,PB{\bm]})\leq{\rm rank}({\bm[}a_1,..,a_{r_A},Pa_1,..,Pa_{r_A},Pb_1,...,Pb_{r-r_A}{\bm ]}) \notag \\
            & \leq {\rm rank}({\bm[}a_1,..,a_{r_A},Pa_1,..,Pa_{r_A}{\bm ]})+r-r_A\notag\\
            &= {\rm rank}({\bm [}a_1,..,a_{r_A},Pa_1,..,Pa_{r_A}{\bm]}\begin{bmatrix}
                I_{r_A} & -I_{r_A}  \\
                0 & I_{r_A}  \\
                \end{bmatrix}) +r-r_A\notag\\
            &\leq r_A+r-r_A+{\rm rank}({\bm [}Pa_1-a_1,..,Pa_{r_A}-a_{r_A} {\bm]}).
        \end{align}
    
Now we denote the cycles in $\pi_P$ with length greater than 1 as $C_1,...,C_{C(\pi_P)}$, and $\zeta_1,...,\zeta_{n-H(\pi_p)}$ as the indexes that are not in any one of $C_1,...,C_{C(\pi_P)}$. We construct a matrix $Y\in \Rbb^{(n+C(\pi_P)-H(\pi_p))\times n}$ as: 
\begin{align*}
  &Y(i,j)=1\text{ if }j=\zeta_i \text{ else } Y(i,j)=0,\ \text{for }i=1,...,(n-H(\pi_p));\\
  &Y(i,j)=1\ \forall j\in C_i, \text{ and } Y(i,j)=0\ \forall j\notin C_i,\\ &\ \ \ \ \ \ \ \ \text{for }i=(n-H(\pi_p)+1),...,(n+C(\pi_P)-H(\pi_p)).\ 
\end{align*} 

It can be verified that 
\begin{align*}
  Y(Pa_i-a_i)=0,\ i=1,...,r_A.
\end{align*}

We denote the null space of $Y$ as Null$(Y)=\{x\in\Rbb^n|Yx=0\}$. From the construction of Y we can see that dim$($Null$(Y))= H(\pi_P)-C(\pi_P)$.  Hence we have 
\begin{align}
  \label{eq:MR3}
  {\rm rank}({\bm [}Pa_1-a_1,..,Pa_{r_A}-a_{r_A} {\bm]})\leq  H(\pi_P)-C(\pi_P).
\end{align}

        On the other hand, we have 
        \begin{align}
            \label{eq:MR4}
            {\rm rank}({\bm [}A,PB {\bm]})\leq {\rm rank}(A)+{\rm rank}(PB)={\rm rank}(A)+{\rm rank}(B)=r_A+r_B.
        \end{align}
        Combining \eqref{eq:MR2}, \eqref{eq:MR3} and \eqref{eq:MR4} , we can obtain \eqref{eq:MR1}.
\end{proof}

Following the proof of Proposition \ref{prop:MR}, it is easy to show the similar result for the case with multiple permutation, which is summarized as the  Corollary \ref{col:multiple_rank}

\begin{corollary}
  \label{col:multiple_rank}
  For the matrix $M={\bm [}A,B_1,..,B_d{\bm ]}\in \Rbb^{n\times m} $ with ${\rm rank}(M)=r$, ${\rm rank}(A)=r_A$, and ${\rm rank}(B_i)=r_{B_i}$, $i=1,...d$, we have $\forall P_1,...,P_d\in\Pc_n$,
  \begin{align}
      {\rm rank}({\bm [}A,P_1B_1,...,P_dB_d{\bm]})\leq \min\{n,m,r_A+\sum_{i=1}^d r_{B_i},r+\sum_{i=1}^d H(\pi_{P_i})-C(\pi_{P_i})\}.
  \end{align} 
\end{corollary}

\begin{proof}[Proof of Proposition \ref{prop:attained_rank}]
  To prove Proposition \ref{prop:attained_rank}, we need an important lemma on measure theory from \citep{halmos2013measure}.
\begin{lemma}
  \label{lm:measure}
  Let $p(x)$ be a polynomial on $\Rbb^n$. If there exists  a $x_0\in\Rbb^n$ such that $p(x_0)\neq 0$, then the Lebesgue measure of the set $\{x|p(x)=0\}$ is 0.
\end{lemma}

   $\forall P\in \Pc_n$, we define the polynomial on $\Rbb^{n\times r}\otimes \Rbb^{r\times m}$ as $$p^r_P(R,E)=\sum_{S\in \Sc_{r}([A,PB])}{\rm det}(S)^2,$$ where det($\cdot$) is the determinant of matrix, and $\Sc_{r}(X)$ is the set of all $r\times r$ sub-matrices in $X$. We denote that $r_P=\min\{2r,r+H(\pi_P)-C(\pi_P)$. We can see that ${\rm rank}({\bm [}A,PB{\bm]})\geq r_P$ if and if only $p^{r_P}_P(R,E)>0$. Therefore, from Lemma \ref{lm:measure} and Proposition \ref{prop:MR} we can conclude that if there exists two matrices $R_0\in \Rbb^{n\times r}$ and $E_0\in \Rbb^{r\times m}$ such that $p^{r_P}_P([R_0,E_0])>0$, then ${\rm rank}({\bm [}A,PB{\bm]})=r_P$ holds with probability 1. In this way, we only need to construct such $R_0$ and $E_0$ for every $P\in \Pc_n$. For simplicity, we denote that $k=H(\pi_p)-C(\pi_P)$. We will discuss how to construct such $R_0$ and $E_0$ for the two cases $0<k\leq n-r$ and $k\geq n-r$, respectively.

  (1) If $0<k\leq n-r$:

  We construct the matrix $Y\in \Rbb^{(n+C(\pi_P)-H(\pi_p))\times n}$ the same way with that in the proof of Proposition \ref{prop:MR}. 
  Firstly, we show that Null$(Y)=$col$(P-I)$.

  col$(P-I)\subseteq$Null$(Y)$: We can verify that $Y(P-I)=0$. 
  
  Null$(Y)\subseteq$col$(P-I)$: This is equivalent to prove that Null$(P-I)\subseteq$col$(Y)$. Now we have $Px=x$, $\forall x\in$Null$(P-I)$. It can be verified that if $Px=x$, then we must have
    $x(s)=x(q)$  if $s$  and $q$ belong to the same cycle  $C_i$, where $C_i$ is one of the cycles in $C_1,...,C_{C(\pi_P)}$. By the definition of $Y$, we can see that  $x\in \text{col}(Y)$.

    Now we know that rank$(P-I)=$dim$(\text{Null}(Y))=k$. We denote the eigen vectors of $P-I$ with non-zero eigen values as $\phi_1,...,\phi_k$, and the eigen vectors with zero eigen values as $\phi_{k+1},...,\phi_n$. Now we have $(P-I)\phi_i=\lambda_i \phi_i$ for $i=1,...,k$ and $(P-I)\phi_i=\lambda_i \phi_i$ for $i=k+1,...,n$.
  
  We construct the matrices $R_0$ and $E_0$ as  \begin{align*}
    &R_0=[\phi_1+\phi_{k+1},\phi_{\min\{2,k\}}+\phi_{k+2},...,\phi_{\min\{r,k\}}+\phi_{k+r}],\\
    &E_0=[I_r,\mathbf{0}_{r\times (m_A-r)},I_r,\mathbf{0}_{r\times (m_B-r)}].
  \end{align*}
Now we have \begin{align*}
  &A=[\phi_1+\phi_{k+1},\phi_{\min\{2,k\}}+\phi_{k+2},...,\phi_{\min\{r,k\}}+\phi_{k+r},\mathbf{0}_{n\times (m_A-r)}],\\
  &B=[\phi_1+\phi_{k+1},\phi_{\min\{2,k\}}+\phi_{k+2},...,\phi_{\min\{r,k\}}+\phi_{k+r},\mathbf{0}_{n\times (m_B-r)}],
\end{align*} since $[A,B]=R_0E_0$.
Therefore, we have
\begin{align*}
  &{\rm rank}({\bm [}A,PB{\bm]})={\rm rank}({\bm[}\phi_1+\phi_{k+1},...,\phi_{\min\{r,k\}}+\phi_{k+r},\lambda_1 \phi_1,...,\lambda_{\min\{r,k\}} \phi_{\min\{r,k\}}{\bm ]})\\
  &={\rm rank}({\bm[}\phi_{k+1},...,\phi_{k+r},\phi_1,..,\phi_{\min\{r,k\}}{\bm ]})\\
  &=r+\min\{k,r\}=\min\{2r,r+k\}.
\end{align*}

Now ${\rm rank}({\bm [}A,PB{\bm]})=r_P$ by this construction of $R_0$ and $E_0$. Hence $p^{r_P}_P([R_0,E_0])>0$.

(2) If $k> n-r$:

We denote that the length of a cycle $C$ as len$(C)$, and denote the cycle with maximum length among the $C_1,...,C_{C(\pi_P)}$ as $C^*$. Now we have \begin{align*}
  {\rm len}(C^*)\geq \frac{H(\pi_P)}{C(\pi_P)}\geq \frac{n}{n-k}>\frac{n}{r}\geq 2r.
\end{align*}
To simplify the notations, we assume that the cycle $C^*$ permute the first $j$ numbers, i.e., $$C^*=(123...(j-2)(j-1)j),$$ where $j> 2r$. We define the vector $u$ as $u=[1,2,3,...,j-2,j-1,j,0,...,0]^\top \in \Rbb^{n}$, and denote the corresponding permutation matrix to  $C^*$ as $P_{*}\in\Pc_n$. We construct the matrices $R_0$ and $E_0$ as  \begin{align*}
  &R_0=\left[
    \begin{matrix}
      u  & P_{*}^2u & \hdots & P^{2r-2}_{*}u\\
     \end{matrix}
     \right],\\
  &E_0=[I_r,\mathbf{0}_{r\times (m_A-r)},I_r,\mathbf{0}_{r\times (m_B-r)}].
\end{align*}
Now we have \begin{align*}
&A=[u  , P_{*}^2u , \hdots , P^{2r-2}_{*}u,\mathbf{0}_{n\times (m_A-r)}],\\
&B=[u  , P_{*}^2u , \hdots , P^{2r-2}_{*}u,\mathbf{0}_{n\times (m_B-r)}].
\end{align*}
Therefore, we have
\begin{align*}
  &{\rm rank}({\bm [}A,PB{\bm]})={\rm rank}({\bm[}u  , P_{*}u , \hdots , P^{2r-1}_{*}u{\bm ]})=2r,
\end{align*} because now ${\bm[}u  , P_{*}u , \hdots , P^{2r-1}_{*}u{\bm ]}$ is a circulant matrix. Now ${\rm rank}({\bm [}A,PB{\bm]})=r_P=2r$ by this construction of $R_0$ and $E_0$. Hence $p^{r_P}_P([R_0,E_0])>0$.
\end{proof}

\begin{proof}[Proof of Theorem \ref{thm:error_bound}.]
  To prove Theorem \ref{thm:error_bound}, we need to derive a series results. We first start with a very important inequality w.r.t nuclear norm.

  \begin{proposition}
    Let $P$ be a permutation matrix, then,
      \label{prop:bound_diff_nuclear}
      \begin{align}
        \label{p:bound_diff_nuclear}
        \|A\|_*+\|B\|_*\geq \|[A,PB]\|_* \geq \frac{\|A\|_*+\|B\|_*}{\|[U_AV_A^\top,PU_BV_B^\top]\|}\geq \frac{\|A\|_*+\|B\|_*}{\sqrt 2}.
      \end{align}
    \end{proposition}

    Based on \eqref{p:bound_diff_nuclear}, the general idea is that under the Assumptions \ref{asp:asp1}, \ref{asp:asp2} and \ref{asp:asp3}, we will have $\|M\|_*\approx \frac{\|A\|_*+\|B\|_*}{\sqrt 2}$ and $\|[U_AV_A^\top,PU_BV_B^\top]\|\to 1$ as $H(\pi_P)$ increases.

    Firstly, we show that under the Assumptions \ref{asp:asp1}, \ref{asp:asp2}, the nuclear norm of the original matrix $M$ will reach the lower bound in \eqref{p:bound_diff_nuclear} approximately, which is summarized as Lemma \ref{lm:core1}.
\begin{lemma}
  \label{lm:core1}
  Under the Assumptions \ref{asp:asp1}, \ref{asp:asp2}, we have 
  \begin{align}
    \|M\|_* \leq (\|A\|_*+\|B\|_*)/\sqrt 2 +(\sqrt 2 +1)\epsilon_1r+\epsilon_2\max\{\|A\|_*,\|B\|_*\}.
  \end{align}
\end{lemma}

Then, we  show that under the Assumptions \ref{asp:asp2}, \ref{asp:asp3}, $\|[U_AV_A^\top,PU_BV_B^\top]\|\to 1$ as $H(\pi_P)$ increases, which is summarized as Lemma \ref{lm:core2}.

\begin{lemma}
  \label{lm:core2}
  Under the Assumptions \ref{asp:asp2}, \ref{asp:asp3}, we have 
  \begin{align}
    \|[U_AV_A^\top,PU_BV_B^\top]\|\leq \sqrt{2-H(\pi_P)\epsilon_3^2/2} + \sqrt{T}\epsilon_2.
  \end{align}
\end{lemma}

    Finally, we  need a classical result on the tail bound for the operator norm of Gaussian matrix, whose proof can be found in \citep{wainwright2019high}.

\begin{lemma}
  \label{lm:tail_bound}
Consider the random matrix $W\in\Rbb^{n\times m}$ whose elements follow $\Nc(0,\sigma^2)$ i.i.d. For any $\delta>0$, we have \begin{align}
  \label{p:tail_bound}
  \|W\|\leq \sqrt L(2+\delta)\sigma
\end{align} holds with probability greater than $1-2\exp\{\frac{-L\delta^2}{2}\}$, where $L=\max\{n,m\}$.
\end{lemma}

  Based on Lemma \ref{lm:tail_bound}, we have 
  \begin{align*}
    \|W\|_*\leq L\|W\|\leq \sqrt{ML\sigma}
  \end{align*} holds with probability greater than $1-2\exp\{-\frac{M}{8L\sigma}\}$.

  From Proposition \ref{prop:bound_diff_nuclear}, Lemma \ref{lm:core1} and Lemma \ref{lm:core2} we can know that, for any $P\in\Pc_n$ with $H(\pi_P)$ satisfies that
  \begin{align*}
    & \frac{M}{\sqrt{2-\frac{H(\pi_p)\epsilon_3}{2}}+\sqrt{T}\epsilon_2}-\|W\|_*
    > \frac{M}{\sqrt{2}}+(\sqrt{2}+1)\epsilon_1r+\epsilon_2N+\|W\|_*,\\
  \end{align*} we must have 
\begin{align*}
  \|A_o,PB_o\|&\geq \|A,PB\| - \|W\|_* \\
    &\geq \frac{M}{\sqrt{2-\frac{H(\pi_p)\epsilon_3^2}{2}+\sqrt{T}\epsilon_2}}-\|W\|_*\\
    & > \frac{M}{\sqrt{2}}+(\sqrt{2}+1)\epsilon_1r+\epsilon_2N+\|W\|_*\\
    &\geq \|A,B\|_*+\|W\|_*\geq \|A_o,B_o\|_*.
\end{align*}

Therefore, with probability greater than $1-2\exp\{-\frac{M}{8L\sigma}\}$, if $H(\pi_P)$ satisfies that \begin{align*}
  \frac{M}{\sqrt{2-\frac{H(\pi_p)\epsilon_3^2}{2}}+\sqrt{T}\epsilon_2}
    > \frac{M}{\sqrt{2}}+(\sqrt{2}+1)\epsilon_1r+\epsilon_2N+2\sqrt{ML\sigma},\tag{@}\label{p:ax6}
\end{align*} we have $\|A_o,PB_o\|>\|A_o,B_o\|_*$. Now we simplify \eqref{p:ax6} as
\begin{align*}
  &\frac{M}{\sqrt{2-\frac{H(\pi_p)\epsilon_3^2}{2}}+\sqrt{T}\epsilon_2}
    > \frac{M}{\sqrt{2}}+(\sqrt{2}+1)\epsilon_1r+\epsilon_2N+2\sqrt{ML\sigma}\\
    \Leftrightarrow\ & \sqrt{2-\frac{H(\pi_p)\epsilon_3^2}{2}} < \frac{\sqrt{2}M}{M+(\sqrt{2}+2)\epsilon_1r+\sqrt{2}\epsilon_2N+2\sqrt{2ML\sigma}} - \sqrt{T}\epsilon_2.
\end{align*}
It can be verified that $$\frac{\sqrt{2}M}{M+(\sqrt{2}+2)\epsilon_1r+\sqrt{2}\epsilon_2N+2\sqrt{2ML\sigma}} - \sqrt{T}\epsilon_2>0$$ from the condition on $\epsilon_1$, $\epsilon_2$ and $\sigma$.

Therefore, we have 
\begin{align*}
 & \sqrt{2-\frac{H(\pi_p)\epsilon_3^2}{2}} < \frac{\sqrt{2}M}{M+(\sqrt{2}+2)\epsilon_1r+\sqrt{2}\epsilon_2N+2\sqrt{2ML\sigma}} - \sqrt{T}\epsilon_2\\
    \Leftrightarrow\ & H(\pi_P)> \frac{2}{\epsilon_3^2}\bigg(2-(\frac{\sqrt{2}M}{M+(\sqrt{2}+2)\epsilon_1r+\sqrt{2}\epsilon_2N+2\sqrt{2ML\sigma}} - \sqrt{T}\epsilon_2)^2 \bigg).
\end{align*}

Since $P^*$ is the optimal solution to \eqref{p:nuclear_norm}, we must have  \begin{align*}
  \|[A_o,P^*\tilde P B_o]\|_*\leq \|[A_o,B_o]\|_*.
\end{align*} Besides,  $P^*\tilde P$ is also a permutation matrix, we denote its corresponding permutation as $\hat \pi$. Now we have 
\begin{align*}
  d_H(\pi_*,\tilde \pi)=H(\hat \pi)\leq \frac{2}{\epsilon_3^2}\bigg(2-(\frac{\sqrt{2}M}{M+(\sqrt{2}+2)\epsilon_1r+\sqrt{2}\epsilon_2N+2\sqrt{2ML\sigma}} - \sqrt{T}\epsilon_2)^2 \bigg).
\end{align*}

\end{proof}

The proof to the auxiliary results used in the proof of Theorem \ref{thm:error_bound} are provided below.

\begin{proof}[Proof of Proposition \ref{prop:bound_diff_nuclear}]
  Since $\|\cdot\|_*$ is a norm, we have \begin{align*}
    \|[A,PB]\|_*=\|[A,\mathbf{0}]+[\mathbf{0},PB]\|_*\leq \|A\|_*+\|PB\|_*=\|A\|_*+\|B\|_*.
  \end{align*}
  Then since $\|\cdot\|_*$ is the dual norm of $\|\cdot\|$, we have 
  \begin{align*}
    \|[A,PB]\|_*&=\sup_{\|Q\|\leq 1} \langle [A,PB],Q \rangle\\
    &\geq \langle [A,PB],\frac{[U_AV_A^\top,PU_BV_B^\top]}{\|[U_AV_A^\top,PU_BV_B^\top]\|} \rangle\\
    &=\frac{\|A\|_*+\|B\|_*}{\|[U_AV_A^\top,PU_BV_B^\top]\|}.
  \end{align*}
  Finally, we have 
  \begin{align*}
    \|[U_AV_A^\top,PU_BV_B^\top]\|&=\sup_{\substack{x\in\Rbb^m\\\|x\|\leq 1}} \|[U_AV_A^\top,PU_BV_B^\top]x\|\\
    &=\sup_{\substack{x_1\in\Rbb^{m_A},x_2\in\Rbb^{m_B}\\\|[x_1^\top,x_2^\top]\|\leq 1}} \|[U_AV_A^\top x_1,PU_BV_B^\top x_2]\|\\
    &\leq \sup_{\substack{x_1\in\Rbb^{m_A},x_2\in\Rbb^{m_B}\\\|[x_1^\top,x_2^\top]\|\leq 1}} \|U_AV_A^\top x_1\|+\|PU_BV_B^\top x_2\|\\
    &\leq \sup_{\substack{x_1\in\Rbb^{m_A},x_2\in\Rbb^{m_B}\\\|[x_1^\top,x_2^\top]\|\leq 1}} \|x_1\|+\|x_2\|=\sqrt{2}.
  \end{align*}
\end{proof}

\begin{proof}[Proof of Lemma \ref{lm:core1}]
  If $r_A\geq r_B$, we have \begin{align*}
    \|M\|_*&=\|[U_A\Sigma_A V_A^\top,U_B\Sigma_B V_B^\top]\|_*\\
    &=\|[U_A\Sigma_A V_A^\top,[u_A^1,...,u_A^T,\mathbf{0},...,\mathbf{0}]\Sigma_B V_B^\top]+\\ &\ \ \ \ \ \ \ \ \ \ [\mathbf{0},[u_A^1-u_B^1,...,u_A^T-u_B^T,u_B^{T+1},...,u_B^r]\Sigma_B V_B^\top]\|_*\\
    &\leq \|[U_A\Sigma_A V_A^\top,[u_A^1,...,u_A^T,\mathbf{0},...,\mathbf{0}]\Sigma_B V_B^\top]\|_*+\\ &\ \ \ \ \ \ \ \ \ \ \|[u_A^1-u_B^1,...,u_A^T-u_B^T,u_B^{T+1},...,u_B^r]\Sigma_B V_B^\top\|_*\\
    &\leq \|[U_A\Sigma_A V_A^\top,[u_A^1,...,u_A^T,\mathbf{0},...,\mathbf{0}]\Sigma_B V_B^\top]\|_*+\epsilon_2 \|B\|_*\\
    &= \|[U_A\Sigma_A V_A^\top,U_A\Sigma_B V_B^\top]\|_*+\epsilon_2 \|B\|_*.\tag{*} \label{p:ax1}
  \end{align*}
  We denote that $trace(\cdot)$ as the trace of matrix. One property of nuclear norm is $$\|A\|_*=trace(\sqrt{AA^\top}).$$
  Then we have  \begin{align*}
    \|[U_A\Sigma_A V_A^\top,U_A\Sigma_B V_B^\top]\|_*&=trace(\sqrt{U_A(\Sigma_A^2+\Sigma_B^2)U_A^\top})\\
&=\sum_{i=1}^r \sqrt{(\sigma_A^{i})^2+(\sigma_B^{i})^2}\\
&\leq \sum_{i=1}^r \frac{\sigma_A^{i}+\sigma_B^{i}}{\sqrt 2}+(\sqrt{(\sigma_A^{i})^2+(\sigma_B^{i})^2}-\frac{\sigma_A^{i}+\sigma_B^{i}}{\sqrt 2})\\
&\leq \sum_{i=1}^r \frac{\sigma_A^{i}+\sigma_B^{i}}{\sqrt 2}+(\sqrt{(\sigma_A^{i})^2+(\sigma_A^{i}+\epsilon_1)^2}-\frac{2\sigma_A^{i}-\epsilon_1}{\sqrt 2})\\
&\leq \frac{\sqrt 2 \epsilon_1 r}{2} + \frac{\|A\|_*+\|B\|_*}{\sqrt 2}+ \\ & \ \ \ \  \ \ \ \ \ \ \ \ \ \ \ \ \ \ \ \ \  \sum_{i=1}^r \frac{2\sigma_A^{i}\epsilon_1+\epsilon_1^2}{\sqrt{2(\sigma_A^{i})^2+2\sigma_A^{i}\epsilon_1+\epsilon_1^2}+\sqrt{2(\sigma_A^{i})^2}}\\
& \leq \frac{\sqrt 2 \epsilon_1 r}{2} + \frac{\|A\|_*+\|B\|_*}{\sqrt 2}+   \sum_{i=1}^r \frac{\sqrt 2 \epsilon_1}{2}+\epsilon_1\\
&=\frac{\|A\|_*+\|B\|_*}{\sqrt 2}+ (\sqrt{2}+1)\epsilon_1 r.\tag{**}
\label{p:ax2}
  \end{align*}

  Combining \eqref{p:ax1} and \eqref{p:ax2}, we have 
  \begin{align*}
    \|[A,B]\|_*\leq \frac{\|A\|_*+\|B\|_*}{\sqrt 2}+ (\sqrt{2}+1)\epsilon_1 r + \epsilon_2 \|B\|_*.
  \end{align*}
Similarly, if $r_B\geq r_A$, we have \begin{align*}
  \|[A,B]\|_*\leq \frac{\|A\|_*+\|B\|_*}{\sqrt 2}+ (\sqrt{2}+1)\epsilon_1 r + \epsilon_2 \|A\|_*.
\end{align*}
Combining them together, we have 
\begin{align*}
  \|[A,B]\|_*\leq \frac{\|A\|_*+\|B\|_*}{\sqrt 2}+ (\sqrt{2}+1)\epsilon_1 r + \epsilon_2 \max \{\|A\|_*,\|B\|_*\}.
\end{align*}
\end{proof}

\begin{proof}[Proof pf Lemma \ref{lm:core2}]
  Firstly, if $r_A\geq r_B$ we have 
  \begin{align*}
    \|[U_AV_A^\top,PU_BV_B^\top]\|&=\|[U_AV_A^\top,P[u_A^1,...,u_A^T,\mathbf{0},...,\mathbf{0}]V_B^\top]\|+\\ &\ \ \ \ \ \ \ \ \ \ \ \ \ \  \|[0,P[u_B^1-u_A^1,...,u_B^T-u_A^T,\mathbf{0},...,\mathbf{0}]V_B^\top]\|\\
    &\leq \|[U_AV_A^\top,P[u_A^1,...,u_A^T,\mathbf{0},...,\mathbf{0}]V_B^\top]\| + \sqrt{T}\epsilon_2. \tag{***} \label{p:ax3}
  \end{align*}
  
To simplify the notations, we denote that $k=H(\pi_P)$ and assume that $\pi_P$ permutes the indexes $(1,...,k)$ into $(\zeta_1,...,\zeta_k)$. Now we have
\begin{align*}
  \langle u_A^{i}, Pu_A^i \rangle &= \sum_{i=1}^k u_A^{i}(i)u_A^{i}(\zeta_i) + \sum_{i=k+1}^n (u_A^{i}(i))^2,
\end{align*} and 
\begin{align*}
  |\sum_{i=1}^k u_A^{i}(i)u_A^{i}(\zeta_i)|&\leq \sum_{i=1}^k | u_A^{i}(i)u_A^{i}(\zeta_i)|\\
  &= \sum_{i=1}^k \frac{(u_A^{i}(i))^2+(u_A^{i}(\zeta_i))^2}{2} -(\frac{(u_A^{i}(i))^2+(u_A^{i}(\zeta_i))^2}{2} - |u_A^{i}(i)u_A^{i}(\zeta_i)|)\\
  &\leq \sum_{i=1}^k (u_A^{i}(i))^2 -( \frac{(u_A^{i}(i))^2+(|u_A^{i}(i)|-\epsilon_3)^2}{2} - |u_A^{i}(i)|(|u_A^{i}(i)|+\epsilon_3))\\
  &=\sum_{i=1}^k (u_A^{i}(i))^2 - (\frac{\epsilon_3^2}{2}+2|u_A^{i}(i)|\epsilon_3)\leq \sum_{i=1}^k (u_A^{i}(i))^2- \frac{\epsilon_3^2}{2}.
\end{align*}
Hence we must have $$|\langle u_A^{i}, Pu_A^i \rangle |\leq 1-\frac{k\epsilon_3^2}{2}.$$

Therefore, we have 
\begin{align*}
    \delta (U_A,P) &\stackrel{\text{def.}}= \max_{\substack{x,y\in \Rbb^T,\\\|x\|=1,\|y\|=1}} \langle [u_A^1,...,u_A^T]x, [Pu_A^1,...,Pu_A^T]y \rangle\\
    &=\max_{\substack{x,y\in \Rbb^T,\\\|x\|=1,\|y\|=1}} \sum_{i=1}^T x(i)y(i) \langle u_A^{i}, Pu_A^i \rangle\\
    &\leq \max_{\substack{x,y\in \Rbb^T,\\\|x\|=1,\|y\|=1}} (1-\frac{k\epsilon_3^2}{2})\sum_{i=1}^T x(i)y(i) \\
    &=1-\frac{k\epsilon_3^2}{2}.
\end{align*}

Now we have, 
\begin{align*}
  &\|[U_AV_A^\top,P[u_A^1,...,u_A^T,\mathbf{0},...,\mathbf{0}]V_B^\top]\|=\sup_{\substack{x\in\Rbb^n,\\\|x\|=1}} \|[U_AV_A^\top,P[u_A^1,...,u_A^T,\mathbf{0},...,\mathbf{0}]V_B^\top]x\|\\
&\leq \sup_{\substack{x_1\in\Rbb^{m_A},x_2\in\Rbb^{m_B}\\\|[x_1^\top,x_2^\top]\|\leq 1}} \sqrt{1+\langle U_AV_A^\top x_1,P[u_A^1,...,u_A^T,\mathbf{0},...,\mathbf{0}]V_B^\top x_2\rangle}\\
&\leq \sup_{\substack{x_1\in\Rbb^{m_A},x_2\in\Rbb^{m_B}\\\|[x_1^\top,x_2^\top]\|\leq 1}} \sqrt{1+\delta(U_A,P)\|x_1\|\|x_2\|}\leq \sqrt{2-\frac{k\epsilon_3^2}{2}}.\tag{****}\label{p:ax4}.
\end{align*}

Combining \eqref{p:ax3} and \eqref{p:ax4}, we have 
\begin{align*}
  \|[U_AV_A^\top,PU_BV_B^\top]\|\leq \sqrt{2-\frac{k\epsilon_3^2}{2}} + \sqrt{T}\epsilon_2.
\end{align*}

The proof is similar for the case $r_B\geq r_A$.
\end{proof}

  \subsection{Dual Problem of (\ref{p:entropyOT})}
  \label{app:dual}
  To simplify the notation, we denote the primal problem as
  \begin{align*}
    \underset{P\in \Pi(\mathbf{1}_n,\mathbf{1}_n)}{\text{minimize }} &\langle C,P\rangle+\epsilon H(P).
  \end{align*}
  We define two dual variables $\alpha,\beta \in \mathbb{R}^{n}$. The Lagrangian function is  
  \begin{align}
    \label{p:lagrange}
    L(P,\alpha,\beta)&=
    \langle C, P\rangle+\epsilon\langle \log{P}-\mathbf{1}_{n\times n},P\rangle+\left\langle\mathbf{1}_{n}-P \mathbf{1}_{n}, \alpha\right\rangle+\left\langle\mathbf{1}_{n}-P^T \mathbf{1}_{n}, \beta\right\rangle.
  \end{align}
  Now we minimize the Lagrangian function w.r.t $P$ (We note that $H(P)$ implicitly imposes that $P\in \mathbb{R}^{n\times n}_+$). From the first-order necessary condition of unconstrainted optimization, we have
  \begin{align}
    C-&\alpha\oplus\beta+\epsilon \log(P)=0,\notag \\ 
    &\Downarrow \notag\\ 
    P=&\text{exp}\bigg\{\frac{\alpha\oplus\beta-C}{\epsilon}\bigg\}.
    \label{p:reconstruct}
  \end{align}
  Substituting it into the Lagrangian function \eqref{p:lagrange} we have the dual objective
  \begin{align*}
    q(\alpha,\beta)=\underset{P}{\min}\ L(P,\alpha,\beta)=\left\langle\mathbf{1}_{n}, \alpha\right\rangle+\left\langle\mathbf{1}_{n}, \beta\right\rangle-\epsilon \bigg\langle \mathbf{1}_{n\times n} , \text{exp}\bigg\{\frac{\alpha\oplus\beta-C}{\epsilon}\bigg\}\bigg\rangle.
  \end{align*}
  Therefore the dual problem is 
  \begin{align}
    \label{p:DualOT}
  \max _{\alpha,\beta \in \mathbb{R}^{n}} \left\langle\mathbf{1}_{n}, \alpha\right\rangle+\left\langle\mathbf{1}_{n}, \beta\right\rangle-\epsilon\bigg\langle  \mathbf{1}_{n\times n} , \text{exp}\bigg\{\frac{\alpha\oplus\beta-C}{\epsilon}\bigg\}\bigg\rangle.
  \end{align}
  We can recover the primal solution $P$ from the dual solution $\alpha$, $\beta$ via \eqref{p:reconstruct}.

  \subsection{A Stable Implementation for Skinhorn Algorithm}
  The Skinhorn algorithm \citep{peyre2019computational} are often used to solve the dual problem \eqref{p:DualOT}, and the standard form of it reads 
  \begin{align*}
    p^{(t+1)}\leftarrow \frac{\mathbf{1}_{n}}{Kq^{(t)}}\text{  and  }q^{(t+1)}\leftarrow \frac{\mathbf{1}_{n}}{K^\top p^{(t+1)}},
  \end{align*}
  where $K=\text{exp}\bigg\{\frac{\alpha\oplus\beta-C}{\epsilon}\bigg\}$, and $p=\exp(\frac{\alpha}{\epsilon})$, $q=\exp(\frac{\beta}{\epsilon})$. If we adopt a small $\epsilon$, the elements of $K$ can overflow to infinity or zero, which causes a numerical issue. We can remedy this by using a different implementation from \citep{peyre2019computational}. 
  \begin{align*}
    \alpha^{(t+1)}&\leftarrow \text{Min}_\epsilon^{\text{row}}(C-\alpha^{(t)}\oplus\beta^{(t)})+\alpha^{(t)},\\
    \beta^{(t+1)}&\leftarrow \text{Min}_\epsilon^{\text{col}}(C-\alpha^{(t+1)}\oplus\beta^{(t)})+\beta^{(t)},
  \end{align*}
  where for any $A\in\mathbb{R}^{n\times m}$, we define the operator $\text{Min}_\epsilon^{\text{row}}$ and $\text{Min}_\epsilon^{\text{col}}$ as
  \begin{align*}
      \operatorname{Min}_{\varepsilon}^{\text {row }}(\mathbf{A}) \stackrel{\text { def. }}{=}\left(\text{min}_{\varepsilon}\mathbf{A}({i, \cdot})\right)_{i} \in \mathbb{R}^{n}, \\
      \operatorname{Min}_{\varepsilon}^{\text {col }}(\mathbf{A}) \stackrel{\text { def. }}{=}\left(\text{min} _{\varepsilon}\mathbf{A}({\cdot, j})\right)_{j} \in \mathbb{R}^{m},
  \end{align*}
  and for any vector $z=[z_1,...,z_n]^\top\in \mathbb{R}^n$, we denote $$\text{min}_\epsilon z\stackrel{\text{def.}}=\min_i z_i-\epsilon \log\sum_j e^{-(z_j-\min_i z_i)/\epsilon}$$ as the $\epsilon$-soft minimum for the elements of $z$.

  \subsection{Relationship between $\text{M}^3\text{O}$ and the {Soft-Impute} Algorithm}
  \label{app:link}
  {Soft-Impute} algorithm \citep{mazumder_spectral_2010} is a classical algorithm for matrix completion. Specifically, it tries to solve the nuclear norm regularized problem
  \begin{align}
    \underset{\widehat M}{\text{minimize }} \frac 12\left\|\Pc_\Omega(X)-\Pc_\Omega(\widehat M) \right\|^2_F+\lambda\left\|\widehat M\right\|_*. 
  \end{align} {Soft-Impute} is a simple iterative algorithm with the following two steps:
  \begin{align}
    \widehat X&\leftarrow \Pc_\Omega(X)+\Pc^\perp_\Omega(\widehat M),\\
    \widehat M&\leftarrow \text{prox}_{\lambda\left\|\cdot\right\|_*}(\widehat X)=U\Sc_\lambda(D)V^\top,
  \end{align} where $\widehat X=UDV^\top$ denotes the singular value decomposition of $\widehat X$, and $\Pc^\perp_\Omega$ is the operator that selects  entries whose indexes are not belonging to $\Omega$. Here $\Sc_\lambda$ is the soft-thresholding operator that operates element-wise on the diagonal matrix $D$, i.e., replacing $D_{ii}$ with $(D_{ii}-\lambda)_+$.
  
  Consider the partial observation extension. For the $\text{M}^3\text{O}$ algorithm, if an exact permutation matrix is obtained, i.e., $\widehat P=\text{exp}\bigg\{\frac{\alpha^*\oplus\beta^*-C(\widehat M_{B})}{\epsilon}\bigg\}\in \Pc_n$, it is easy to verify that the the gradient in Algorithm \ref{alg:McubicO} has the following form,
  \begin{align*}
    \nabla_{\widehat M}F_\epsilon(\widehat M,\alpha^*,\beta^*)=2(\Pc_\Omega(\widehat M)-\Pc_\Omega([A,\widehat P\tilde B])).
  \end{align*}
  In this way, if we adopts $\rho_k=0.5$, the proximal gradient update becomes 
  \begin{align*}
    \widehat M^{\text{k+1}} \leftarrow \text{prox}_{\lambda\left\|\cdot\right\|_*} (\Pc_\Omega([A,\widehat P\tilde B])+\Pc^\perp_\Omega(\widehat M^k)).
  \end{align*}
In practice, $\widehat P$ often becomes very close to an exact permutation matrix and the stepsize often reaches the upper bound 0.5, when the algorithm is close to convergence. In this scenario, our algorithm becomes equivalent to the Soft-Impute algorithm. Therefore, we adopt  the {Soft-Impute} algorithm as a baseline method for matrix completion without correspondence issue.

\subsection{$\text{M}^3\text{O}$-AS-DE for the d-correspondence Problem}
\label{app:algo}
In this section, we summarize our proposed algorithm {$\text{M}^3\text{O}$-AS-DE} for the  general d-correspondence problem \eqref{p:multiple} in Algorithm \ref{alg:M30general}.
To determinate the stop of the Max-Oracle, we find that the criterion $$\frac{1}{\sqrt n}\left\|\mathbf{1}_n^\top\widehat P-\mathbf{1}_n^\top\right\|_2\leq\varepsilon$$ works well in practice, which serves as a good indicator for the $\varepsilon$-good optimality.
\begin{algorithm}
  \small
  \SetKwFor{DoParallel}{for}{do in parallel}{end}
  \caption{{$\text{M}^3\text{O}$-AS-DE}}
  \label{alg:M30general}
  \KwIn{stepsize parameter $\omega$, number of correspondence $d$, number of iterations $N$, number of tolerance steps $K$, initial entropy coefficient $\epsilon$, tolerance $\varepsilon$, observation matrix $ M_o=[A_o, B_o^1,...,B_o^d]$, initial matrix $\widehat M=[\widehat M_A,\widehat M_{B_1},...,\widehat M_{B_d}]$, nuclear norm coefficient $\lambda$, the set of observable indexes  $\Omega$.}
  Initialize $\widehat P_\text{new}^l=\mathbf{0}_{n\times n}$ for $l=1,...,d$.

  \For{$k=1:N$}{

\DoParallel{$l=1:d$}{
  $\widehat P_\text{old}^l=\widehat P_\text{new}^l$.

$\hat \alpha^l=\hat\beta^l=\mathbf{1}_{n}$.

Compute the partial pairwise cost matrix $C(\widehat M_{B_l})$.

\While{$\frac{1}{\sqrt n}\left\|\mathbf{1}_n^\top\widehat P-\mathbf{1}_n^\top\right\|_2>\epsilon$}{
	
$\hat\alpha^l\leftarrow \text{Min}_\epsilon^{\text{row}}(C(\widehat M_{B_l})-\hat\alpha^l\oplus\hat\beta^l)+\hat\alpha^l$,

	$\hat\beta^l\leftarrow \text{Min}_\epsilon^{\text{col}}(C(\widehat M_{B_l})-\hat\alpha^l\oplus\hat\beta^l)+\hat\beta^l$,

  $\widehat P_\text{new}^l\leftarrow \text{exp}\bigg\{\frac{\hat\alpha^l\oplus\hat\beta^l-C(\widehat M_{B_l})}{\epsilon}\bigg\}$.

}
Compute the stepsize $\rho_l$ as discussed in Section \ref{sec:algorithm}.

$\widehat M_{B_l} \leftarrow \widehat M_{B_l}-\rho_l\nabla_{\widehat M} F^l_\epsilon(\widehat M_{B_l},\alpha^l,\beta^l),$ where $$F^l_\epsilon(\widehat M_{B_l},\alpha,\beta)\stackrel{\text{def.}}=\left\langle\mathbf{1}_{n}, \alpha\right\rangle+\left\langle\mathbf{1}_{n}, \beta\right\rangle-\epsilon\bigg\langle  \mathbf{1}_{n\times n} , \text{exp}\bigg\{\frac{\alpha\oplus\beta-C_\Omega(\widehat M_{B_l})}{\epsilon}\bigg\}\bigg\rangle.$$
}
$\widehat M_A\leftarrow \Pc_\Omega(A)+\Pc^\perp_\Omega(\widehat M_A)$.

  $\widehat M \leftarrow \text{prox}_{\lambda\left\|\cdot\right\|_*} ([\widehat M_A,\hat M_{B_1},...,\widehat M_{B_d}])).$
  
  \If{the objective value is not improved over $K$ steps}{
    $\epsilon\leftarrow \epsilon/2.$
  }{}
  
  }

  \end{algorithm}

  \subsection{The Baseline Algorithm}
  \label{app:BCD}
  We also extend the Baseline algorithm to a similar d-correspondence problem as \eqref{p:multiple}. Specifically, the extended Baseline algorithm tries to solve the unsmoothed problem  \begin{align}
    \label{p:multiplt_BCD}
    \min_{\widehat M}\min_{P_1,...,P_d}&\left\|\Pc_{\Omega}(A_o) - \Pc_{\Omega}(\widehat M_A)\right\|_F^2+\sum_{l=1}^{d}\langle C(\widehat M_{B_l}), P_l \rangle+\lambda\left\|\widehat M\right\|_*,\\
    &\text{s.t. }P_l \in \Pc_n,\text{ for }l=1,...,d.\notag
  \end{align}  We summarize the  algorithm in Algorithm \ref{alg:BCD}. 
  \begin{algorithm}[htbp]
    \small
    \caption{Baseline}\label{alg:BCD}
    \SetKwFor{DoParallel}{for}{do in parallel}{end}
    \KwIn{number of iterations $N$, number of Proximal Gradient iterations $N_p$, tolerance $\varepsilon$, observation matrix $ M_o=[A_o, B_o^1,...,B_o^d]$, initial matrix $\widehat M=[\widehat M_A,\widehat M_{B_1},...,\widehat M_{B_d}]$, nuclear norm coefficient $\lambda$, partial observation operator $\Pc_\Omega$.}
    \For{$k=1:N$}{
      \DoParallel{$l=1:d$}{
     Solving the inner problem of (\ref{p:multiplt_BCD}) for $\hat P^l$ up to tolerance $\varepsilon$ via Hungarian algorithm.     
     }
     $ X\leftarrow [A_o,\hat P^1 B_o^1,...,\hat P^dB_o^d]$.

     \For{$i=1:N_p$}{
      $\hat X\leftarrow \Pc_\Omega(X)+\Pc^\perp_\Omega(\hat M)$,
  
      $\hat M\leftarrow \text{prox}_{\lambda\left\|\cdot\right\|_*}(\hat X)$.
      }}
    \end{algorithm}

    \subsection{The MUS Algorithm}
    \label{app:adapt}
    In this section, we provide details for the MUS algorithm discussed in the Section \ref{sec:experiment}. Firstly, inspired by \citep{yao2021unlabeled}, we first transform the MRUC problem, i.e, to recover $[A,B]$ from $[A,\tilde PB]$, into a MUS problem as follows,
    \begin{align}
      \label{p:MUSadapt}
      \min_{P\in\Pc_n,W\in\Rbb^{m_B\times m_A}} \|A-P\tilde PBW\|_F^2.
    \end{align}
    Then, for the scenario without multiple correspondence and missing values, we adopt the algorithm in \citep{zhang2020optimal} to solve \eqref{p:MUSadapt}.

    To extend it into the d-correspondence problem considered by \eqref{p:multiple}, we adopt tow simple procedures. Specifically, to deal with the missing value, we first fill in the missing entries of each submatrices using the Soft-Impute algorithm. As for the multiple correspondence issue, we simply run the MUS algorithm in multiple times. For example, if we want solve the d-correspondence problem, we typically apply the MUS algorithm to the following series of problems in turn,
    \begin{align*}
      \min_{P\in\Pc_n,W\in\Rbb^{m_B\times m_A}} \|A_o-PB_o^lW\|_F^2,\ l=1,...,d.
    \end{align*}



\subsection{Details for the experiments}
\label{app:exp}
We use Matlab 2020b for the numerical experiments. The computer environment consists of Intel i9-10920x for CPU and 32GB RAM.
\subsubsection{Hyperparameters setting}
\textbf{Simulated data. }
We adopt fixed nuclear norm coefficient $\lambda$ in the experiments on simulated data. Specifically, for each setting, we choose the best $\lambda$ out of three candidate values that are 0.4, 0.5 and 0.6. Since adopting large $\omega$ will preserve the final performance and only degrade the convergence speed, we take $\omega=3$ for all the experiments. For the tolerance of Sinkhorn algorithm, we take $\varepsilon=0.01$ for all the experiments.

\textbf{MovieLens 100K. } For all the algorithms, we adopt a sequence of values for $\lambda$. Specifically, we start the algorithm with $\lambda=300$, and once the algorithm stops improving the objective function for 10 steps, we shrink the value as $\lambda\leftarrow \lambda-10$ until $\lambda$ becomes lower than 10. We take $\omega=0.5$ for all the experiments and also set the tolerance of Sinkhorn algorithm as $\varepsilon=0.01$. 

\subsubsection{Numbers of Skinhorn Iteration}
Typically, the numbers of Skinhorn iteration  required to retrieve an $\varepsilon$-good solution mainly depends on the entropy coefficient $\epsilon$. This also implies that the decaying entropy regularization strategy can also accelerate the convergence process. Figure \ref{fig:niter} shows the relationship between the  numbers of Skinhorn iteration and entropy coefficient $\epsilon$ under the same simulated data setting with Figure \ref{fig:exp1}. The dash lines and intervals reflect mean, min, maximum aggregated from 20 independent trials. For a practical implementation, we restrict the maximum numbers of Skinhorn iteration to 10000 on the numerical experiments.
\begin{figure}[htbp]
  \centering
  \includegraphics[width=6cm]{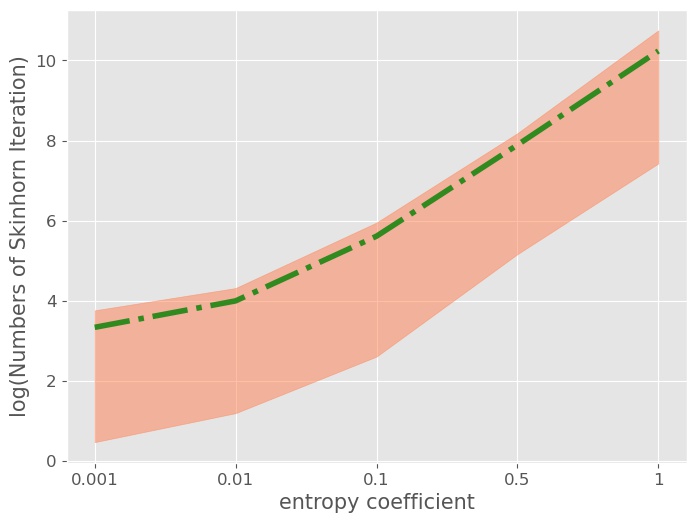}
  \caption{The required numbers of Skinhorn iteration v.s. entropy coefficient $\epsilon$}
  \label{fig:niter}
\end{figure}

\subsubsection{Problem formulation for the face recovery problem}
In the face recovery experiment, the cost matrix $C$ is constructed as $$C(i,j)=\|P_{\Omega}(B(i)-\widehat{M}(j))\|_F^2,$$ where $B(1),...,B(13)\in\Rbb^{30\times 30}$ are the shuffled pixel blocks from the upper left of the corrupted image shown in Figure \ref{fig:perm_face}, and $\widehat{M}(1),...,\widehat{M}(13)\in\Rbb^{30\times 30}$ are the corresponding recovered pixel blocks from the upper left of the current recovered image. 

We choose fixed stepsize $\rho_k=0.1$, and choose the initial entropy coefficient as $\epsilon=100$. To obtain the initial matrix $\widehat M$, we first complete each pixel blocks independently using the Soft-Impute algorithm. We denote the filled matrix as $M_1$, and carry out the singular decomposition of it as $M_1=\sum_i \sigma_i u_i v_i^\top$. Then we set the initial matrix as $\widehat M=\sigma_1 u_1 v_1^\top$. 

More results similar to Figure \ref{fig:face} are shown in Figure \ref{fig:moreface}.

\begin{figure}[htbp]
  \centering
  \subfigure[ Original]{
  \includegraphics[width=3cm]{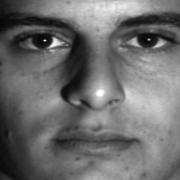}
  }
  \quad
  \subfigure[Corrupted]{
  \includegraphics[width=3cm]{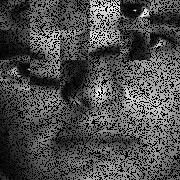}
  }
  \quad
  \centering
  \subfigure[Baseline]{
  \includegraphics[width=3cm]{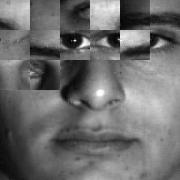}
  }
  \quad
  \subfigure[$\text{M}^3\text{O}$]{
  \includegraphics[width=3cm]{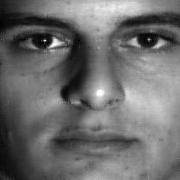}
  }
  \quad
  \centering
  \subfigure[ Original]{
  \includegraphics[width=3cm]{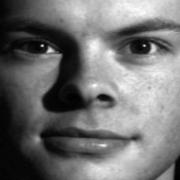}
  }
  \quad
  \subfigure[Corrupted]{
  \includegraphics[width=3cm]{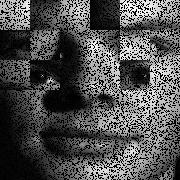}
  }
  \quad
  \centering
  \subfigure[Baseline]{
  \includegraphics[width=3cm]{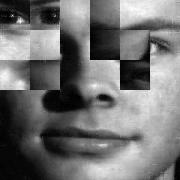}
  }
  \quad
  \subfigure[$\text{M}^3\text{O}$]{
  \includegraphics[width=3cm]{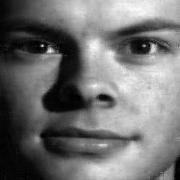}
  }
  \quad
  \centering
  \subfigure[ Original]{
  \includegraphics[width=3cm]{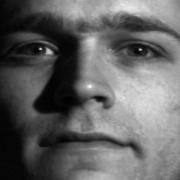}
  }
  \quad
  \subfigure[Corrupted]{
  \includegraphics[width=3cm]{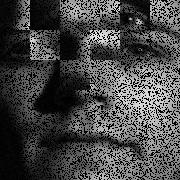}
  }
  \quad
  \centering
  \subfigure[Baseline]{
  \includegraphics[width=3cm]{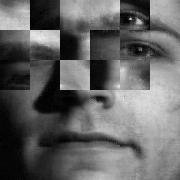}
  }
  \quad
  \subfigure[$\text{M}^3\text{O}$]{
  \includegraphics[width=3cm]{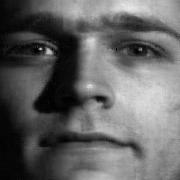}
  }
  \caption{\small Performance of {{M}$^3$O} on more face images from Yale B database.}
  \label{fig:moreface}\vspace{-0.5cm}
  \end{figure}